%% file: ConvexOptwBandits-arxiv.tex
\documentclass[prodmode]{ec-acmsmall}

\usepackage[ruled]{algorithm2e}
\renewcommand{\algorithmcfname}{ALGORITHM}
\SetAlFnt{\small}
\SetAlCapFnt{\small}
\SetAlCapNameFnt{\small}
\SetAlCapHSkip{0pt}
\IncMargin{-\parindent}

\acmVolume{X}
\acmNumber{X}
\acmArticle{X}
\acmYear{2014}
\acmMonth{2}

\usepackage[numbers]{natbib}

\usepackage[dvipsnames,usenames]{xcolor}
\usepackage{prettyref}
\usepackage[letterpaper=true,colorlinks=true,pdfpagemode=none,urlcolor=Blue,linkcolor=RoyalBlue,citecolor=OliveGreen,pdfstartview=FitH]{hyperref}

\newtheorem{assumption}{Assumption}
\usepackage{algorithmic}
\usepackage{amssymb}
\usepackage{amsmath}

\newcommand{\Ex}{\mathbb{E}}
\newcommand{\OPT}{\text{OPT}}
\newcommand{\ALGO}{\text{ALGO}}
\newcommand{\rad}{{\rm rad}}
\newcommand{\crad}{{\gamma}}
\newcommand{\epsilonVal}{\sqrt{\frac{\crad m}{B}} + \log(T) \frac{\crad  m}{B}}
\newcommand{\cradVal}{\log(\frac{mTd}{\delta})}
\newcommand{\comment}[1]{}
\newcommand{\reg}{{\text{regret}}}
\newcommand{\areg}{{\text{avg-regret}}}
\newcommand{\LP}{\text{LP}}
\newcommand{\UCB}{\text{UCB}}
\newcommand{\LCB}{\text{LCB}}
\newcommand{\diff}{{\cal Q}}

\newcommand{\bs}[1]{\boldsymbol{#1}}
\newcommand{\sumT}{{\textstyle{\sum}}}
\newcommand{\BwK}{BwK}
\newcommand{\BwC}{BwC}
\newcommand{\BwR}{BwR}
\newcommand{\BwCR}{BwCR}
\newcommand{\MAB}{MAB}
\newcommand{\OCO}{OCO}

\newcommand{\tarm}{i_t}
\newcommand{\sarm}{i_s}
\newcommand{\cx}{\bs{b}}
\newcommand{\wt}{\bs{w}}
\newcommand{\B}{\bs{A}}
\newcommand{\q}{\bs{q}}

\newcommand{\rs}{r}
\newcommand{\mrv}{\boldsymbol{\mu}}
\newcommand{\mrvS}{{\mu}}
\newcommand{\Bcv}{\boldsymbol{c}}
\newcommand{\BcvS}{{c}}
\newcommand{\Bmcm}{\boldsymbol{C}}
\newcommand{\BmcmS}{C}
\newcommand{\cv}{\boldsymbol{v}}
\newcommand{\cvS}{{v}}
\newcommand{\mcm}{\boldsymbol{V}}
\newcommand{\mcmS}{{V}}
\newcommand{\mcme}{\boldsymbol{\tilde{V}}}
\newcommand{\mcmeS}{{\tilde{V}}}

\newcommand{\rv}{\boldsymbol{v}}

\newcommand{\mrm}{\boldsymbol{V}}

\newcommand{\mrme}{\boldsymbol{\tilde{U}}}

\newcommand{\dis}{\boldsymbol{p}}
\newcommand{\disS}{{p}}
\newcommand{\optdis}{\boldsymbol{p}^*}

\newcommand{\thetaV}{\boldsymbol{\theta}}
\newcommand{\phiV}{\boldsymbol{\phi}}

\newcommand{\HC}{{\cal H}}
\newcommand{\cHypercube}{\HC}

\newcommand{\x}{\bs{x}}
\newcommand{\y}{\bs{y}}
\newcommand{\z}{\bs{z}}
\newcommand{\g}{\bs{g}}
\newcommand{\s}{\bs{s}}
\newcommand{\avgt}[1]{\bar{#1}_{1:t}}
\newcommand{\avgtp}[1]{\bar{#1}_{1:t+1}}
\newcommand{\avgtm}[1]{\bar{#1}_{1:t-1}}
\newcommand{\avgT}[1]{\bar{#1}_{1:T}}
\newcommand{\avgxvt}{\avgt{\x}}
\newcommand{\avgxvtp}{\avgtp{\x}}

\newcommand{\xv}{\x}  
\newcommand{\optxv}{\x^*} 

\newcommand{\normOned}{||{\bf 1}_d||}

\newcommand{\proj}{\pi}

\newcommand{\vertex}[1]{\bs{Z}_t(#1)}

\newcommand{\commentShipra}[1]{}%{\textcolor{blue}{#1}}
\newcommand{\regOCO}{{\cal R}^{\text{c}\texttt{}}}
\newcommand{\regB}{\cal{R}(T,\delta)}
\newcommand{\emp}[1]{\widehat{#1}}

\newcommand{\sC}{\beta}
\newcommand{\sP}{\eta}

\newcommand{\threepartdef}[6]
{
	\left\{
		\begin{array}{ll}
			#1 & \text{if } #2 \\
			#3 & \text{if } #4 \\
			#5 & \text{if } #6
		\end{array}
	\right.
}

\newcommand{\EQ}[3]{\begin{center} \vspace{#1}$#3$ \vspace{#2}\end{center}}	
\newcommand{\LEQ}[4]{\begin{equation}\label{#3}\vspace{#1}\textstyle{#4} \vspace{#2}\end{equation}}
\begin{document}

% Page heads
\markboth{Agrawal and Devanur}{Bandits with concave rewards and convex knapsacks}

% Title portion
%\title{Convex Optimization with Bandits\\
%Or
\title{Bandits with concave rewards and convex knapsacks}
\author{SHIPRA AGRAWAL
\affil{Microsoft Research} 
NIKHIL R. DEVANUR
\affil{Microsoft Research} 
}
% NOTE! Affiliations placed here should be for the institution where the
%       BULK of the research was done. If the author has gone to a new
%       institution, before publication, the (above) affiliation should NOT be changed.
%       The authors 'current' address may be given in the "Author	's addresses:" block (below).
%       So for example, Mr. Abdelzaher, the bulk of the research was done at UIUC, and he is
%       currently affiliated with NASA.

\begin{abstract}
In this paper, we consider a very general model for exploration-exploitation tradeoff which allows arbitrary concave rewards and convex constraints on the decisions across time, in addition to the customary limitation on the time horizon. This model subsumes the classic multi-armed bandit (\MAB) model, and the Bandits with Knapsacks (\BwK) model of \citet{BwK}. We also consider an extension of this model to allow linear contexts, similar to the linear contextual extension of the \MAB~model. We demonstrate that a natural and simple extension of the \UCB~family of algorithms for \MAB~provides a polynomial time algorithm that has near-optimal regret guarantees for this substantially more general model, and matches the bounds provided by \citet{BwK} for the special case of \BwK, which is quite surprising. We also provide computationally more efficient algorithms by establishing interesting connections between this problem and other well studied problems/algorithms such as 
the Blackwell approachability problem, online convex optimization, and the Frank-Wolfe technique for convex optimization.

We give examples of several concrete applications, where this more general model of bandits allows for richer and/or more efficient formulations of the problem.
\end{abstract}

%\category{C.2.2}{Computer-Communication Networks}{Network Protocols}
%
%\terms{Design, Algorithms, Performance}
%
%\keywords{}
%
%\acmformat{Gang Zhou, Yafeng Wu, Ting Yan, Tian He, Chengdu Huang, John A. Stankovic,
%and Tarek F. Abdelzaher, 2010. A multifrequency MAC specially
%designed for  wireless sensor network applications.}

% At a minimum you need to supply the author names, year and a title.
% IMPORTANT:
% Full first names whenever they are known, surname last, followed by a period.
% In the case of two authors, 'and' is placed between them.
% In the case of three or more authors, the serial comma is used, that is, all author names
% except the last one but including the penultimate author's name are followed by a comma,
% and then 'and' is placed before the final author's name.
% If only first and middle initials are known, then each initial
% is followed by a period and they are separated by a space.
% The remaining information (journal title, volume, article number, date, etc.) is 'auto-generated'.

%\begin{bottomstuff}
%This work is supported by the National Science Foundation, under
%grant CNS-0435060, grant CCR-0325197 and grant EN-CS-0329609.
%
%Author's addresses: G. Zhou, Computer Science Department,
%College of William and Mary; Y. Wu  {and} J. A. Stankovic,
%Computer Science Department, University of Virginia; T. Yan,
%Eaton Innovation Center; T. He, Computer Science Department,
%University of Minnesota; C. Huang, Google; T. F. Abdelzaher,
%(Current address) NASA Ames Research Center, Moffett Field, California 94035.
%\end{bottomstuff}

\maketitle

\input{introduction}

\input{prelims} 
\input{applications}

\input{ucb} 
\input{efficient}

\input{efficient-BwCR}
% Start of "Sample References" section

% Acknowledgments
%\begin{acks}

%\end{acks}

% Bibliography
\bibliographystyle{acmsmall}
{\small
\bibliography{bibliography_knapsack}
}
                             % Sample .bib file with references that match those in
                             % the 'Specifications Document (V1.5)' as well containing
                             % 'legacy' bibs and bibs with 'alternate codings'.
                             % Gerry Murray - March 2012

% History dates
%\received{February 2007}{March 2009}{June 2009}
% Appendix
\appendix
\input{appendix}
\setcounter{section}{1}

% Electronic Appendix
%\elecappendix

\end{document}

%% file: introduction.tex
\section{Introduction}

Multi-armed bandit (henceforth, \MAB) is a classic model for handling exploration-exploitation tradeoff inherent in many sequential decision making problems.  \MAB~ algorithms have found a wide variety of applications in clinical trials, web search, internet advertising, multi-agent systems, queuing and scheduling etc.
%\textcolor{red}{list of applications of \MAB algorithms here} 
The classic \MAB~ framework however only handles ``local'' constraints and ``local'' rewards: 
the constraint is only on the decision in each step and the total reward is necessarily a summation of the 
rewards in each step. (The only constraint allowed on decisions accross time is a bound on the number of trials.) 
For many real world problems there are multiple complex constraints on resources that are consumed during the entire decision process. 
Further, in some applications it may be desirable to evaluate the solution not simply by the sum of rewards obtained at individual time steps, but by a more complex utility function. 
We illustrate several such example scenarios in our Applications section (Section \ref{sec:applications}). 
This paper, in succession to the %\commentShipra{removed:building on top of} 
recent results by \citet{BwK}, extends the \MAB~ framework to handle very general 
``global'' constraints and rewards. 

%Recently, there has been some progress in extending the \MAB model to resolve some of these difficulties.
\citet{BwK} took the first step in this direction by successfully extending the \MAB~ model to include linear knapsack constraints on the resources consumed over time. In their model, which they call Bandits with Knapsacks (\BwK), decision at any time $t$ results in a reward and a $d$-dimensional resource consumption vector, and there is a pre-specified budget representing the maximum amount of each resource that can be consumed in time $t$. \citet{BwK} combine techniques from UCB family of algorithms for \MAB, and techniques from online learning algorithms in a non-trivial manner to provide an algorithm with near-optimal regret guarantees for this problem.

In this paper, we introduce a substantial generalization of the \BwK~setting, to include arbitrary concave rewards and arbitrary convex constraints. 
In our vector-valued bandit model, decision at any time $t$ results in the observation of a $d$-dimensional vector $\cv_t$. There is a prespecified convex set $S$ and a prespecified concave obective function $f$, and the goal is that the average of the observed vectors in time $T$ belongs to the specified convex set while maximizing the concave objective. This is essentially the most general convex optimization problem. We refer to this model as ``Bandits with Convex knapsacks and concave Rewards" (henceforth, \BwCR). We also consider an extension of \BwCR~to allow contexts, simiar to the linear contextual bandits extesion of \MAB~\cite{Chu2011}. \BwCR~subsumes \BwK~as a special case when the convex set is simply given by the knapsack constraints, and the objective function is linear. We discuss applications in several domains such as sensor measurements, network routing, crowdsourcing, pay-per-click advertising, which substantially benefit from the more general \BwCR~framework -- either by admitting richer models, or by more efficient formulation of existing models.

Another important contribution of this paper is to demonstrate that a conceptually simple 
 and natural extension of the UCB family of algorithms for \MAB~\cite{Auer2002, Auer2003} provides near-optimal regret bounds for this substantially more general \BwCR~setting, and even for the contextual version of \BwCR. 
Even in the special case of \BwK, this natural extension of UCB algorithm achieves regret bounds matching the problem-dependent lower (and upper) bounds provided by \citet{BwK}. 
This is quite surprising and  is in contrast to the discussion in \citet{BwK}, 
where the need for special techniques for this problem was emphasized, in order to achieve sublinear regret. 

However, this natural extension of the UCB algorithm for \BwCR, even though polynomial-time implementable (as we show in this paper), may not be very computationally efficient. For example, our UCB algorithm for the special case of \BwK~requires solving an \LP~ with $m$ variables and $d$ constraints at every time step. In general, we show that one would require solving a convex optimization problem by ellipsoid method at every time step, for which computing separating hyperplanes itself needs 
another application of the ellipsoid algorithm. 
%, where as the primal-dual algorithm for \BwK in \cite{BwK} required only $O(md)$ computation in every step. 

Our final contribution is giving computationally more efficient algorithms by establishing (sometimes surprising) 
connections between the \BwCR~problem and other well studied problems/algorithms such as 
the Blackwell approachability problem \cite{blackwell1956}, online convex optimization \cite{Zinkevich03}, 
and the Frank-Wolfe (projection-free) algorithm for convex optimization \cite{frank-wolfe}.  
We provide two efficient algorithms, a ``primal'' algorithm based on the Frank-Wolfe algorithm and 
a ``dual'' algorithm based on the reduction of Blackwell approachability to online convex optimization \cite{blackwell2011}.
One may be faster than the other depending on the properties of the objective function $f$ and convex set $S$.  
As an aside, the primal algorithm establishes a connection between Blackwell's algorithm for the approachability problem and the Frank-Wolf algorithm. The dual algorithm turns out to be almost identical to the primal-dual algorithm (PD-BwK) of \citet{BwK} for the special case of \BwK~problem. 

%In this paper, we make a thorough investigation into this matter, and provide a general reduction from the UCB based algorithm for \BwCR to fast algorithms. Using ideas from algorithms for Blackwell approachability problem \cite{blackwell}, the reduction of former to online convex optimization \cite{blackwell2011}, and from Frank-Wolfe technique for convex optimization \cite{frank-wolfe}, we provide two alternative approaches -- primal and dual -- for efficient implementation of the UCB algorithm, one may be faster than the other depending on the properties of objective function $f$ and convex set $S$.  

%Our findings also give an alternate derivation of the primal-dual algorithm (PD-BwK) of \citet{BwK} for the special case of \BwK problem. And, as a side result, we establish a connection between Blackwell algorithm for the approachability problem and Frank-Wolf projection-free technique for convex optimization. %This may be of independent interest.

%% file: prelims.tex
\section{Preliminaries and main results}

\subsection{Bandit with knapsacks (\BwK)}
\label{prelim:BwK}
The following problem was called Bandit with Knapsacks (\BwK) by \citet{BwK}. 
There is a fixed and known finite set of $m$ arms (possible actions), available to the learner, henceforth called the algorithm. There are $d$ resources and finite time-horizon $T$, where $T$ is known to the algorithm. In each time step $t$, the algorithm plays an arm $\tarm$ of the $m$ arms, receives reward $\rs_t \in [0,1]$, and consumes amount $\BcvS_{t,j} \in [0,1]$ of each resource $j$. The reward $\rs_t$ and consumption $\Bcv_{t} \in \mathbb{R}^d$ are revealed to the algorithm after choosing arm $\tarm$. The rewards and costs in every round are generated i.i.d. from some unknown fixed underlying distribution. More precisely,
there is some fixed but unknown $\mrv \in \mathbb{R}^m, \Bmcm \in \mathbb{R}^{d\times m}$ such that 
\EQ{-0.05in}{-0.05in}{\Ex[\rs_{t}| \tarm] = \mrvS_{\tarm},  \ \ \ \Ex[\BcvS_{t,j}(t) | \tarm] = \BmcmS_{j, \tarm}.}
%Define history ${\cal H}_{t-1}$ to include all the choices of plays, rewards and consumptions until round time $t-1$, i.e.
%$${\cal H}_{t-1} = \{a(\tau), r_{a(\tau)}(\tau), c_{a(\tau),j}(\tau), j=1,\ldots, d, \tau=1\ldots, t-1\}.$$
In the beginning of every time step $t$, the algorithm needs to pick $i_t$, using only the history of plays and outcomes until time step $t-1$. %the information only in ${\cal H}_{t-1}$.
There is a hard constraint of $B_j$ on the resource consumption of every $j$. The algorithm stops at the earliest time $\tau$  %\textcolor{red}{Why is this $T+1$? It's a little confusing.} 
 when one or more of the constraints is violated, i.e. if $\sum_{t=1}^{\tau} \BcvS_{t,j}(t) > B_j$ for some $j$, or if the time horizon ends, i.e. $\tau > T$.
Its total reward is given by the sum of rewards in all rounds preceding $\tau$, i.e. $\sum_{t=1}^{\tau-1} \rs_{t}$. The goal of the algorithm is to maximize the expected total reward.
The values of $B_j$ are known to the algorithm, and without loss of generality we can assume $B_j=B=\min_j B_j$ for all $j$. (Multiply each $\BcvS_{t,j}$ by $B/B_j$.) 

\paragraph{Regret and Benchmark} Regret is defined as the difference in the total reward obtained by the algorithm and $\OPT$, where $\OPT$ denotes the total expected reward for the optimal dynamic policy. %Let $\Gamma \le T+1$ denote the time at which algorithm stops (due to constraint violation), then, regret in time $T$ is defined as
\LEQ{0in}{0in}{def:regret}{\reg(T) = \OPT -\sum_{1\le t<\tau} \rs_t.}
For any $\mrv, \Bmcm$, let $\LP(\mrv, \Bmcm)$ denote the value of the following linear program.
\begin{equation}
\label{def:LP}
\begin{array}{lcl}
\max_{\dis} & \mrv\cdot \dis & \\
{\rm s.t.} & \Bmcm \dis \preceq \frac{B}{T}{\bf 1}, &\\
 & \dis \in \Delta_m & 
\end{array}
\end{equation}
where $\Delta_m$ denotes the $m$-dimensional simplex, i.e., $\Delta_m=\{\dis:\sum_{i=1}^m p_i=1, p_i\ge 0, i=1,\ldots, m\}$, and,  $\preceq, \succeq$ denote component-wise $\le$ and $\ge$ respectively. 
%We also use $\LP(\mrv, \Bmcm, \epsilon)$ to denote the value of this linear program when clear from the context.
It is easy to show  that $\LP(\mrv, \Bmcm) \ge \frac{\OPT}{T}$. (For example, see \citet{Devanur2011}, or Lemma 3.1 of \citet{BwK}.) Hence $T\cdot\LP(\mrv, \Bmcm)$ is commonly used in place of $\OPT$ in the analysis of regret.  
%Also, it is easy to show that for any $\epsilon\in [0,1]$, $\LP(\mrv, \Bmcm, \epsilon) \ge (1-\epsilon) \LP(\mrv, \Bmcm, 0)$, so that 
%\begin{equation}
%\label{eq:benchmark}
%\LP(\mrv, \Bmcm, \epsilon) \ge (1-\epsilon)\LP(\mrv, \mcm, 0) \ge (1-\epsilon)\frac{\OPT}{T}.
%\end{equation}
%Therefore, for any $\epsilon\in [0,1]$, regret
%\begin{equation}
%\reg(T) \le T \cdot \LP(\mrv, \Bmcm, \epsilon) - \sum_{t=1}^{\Gamma-1} \rs_t + \epsilon \OPT. 
%\end{equation}

% In the next section we define \BwCR by generalizing the linear knapsack constraints and linear reward objective to convex constraints and concave rewards respectively. 

%%%%%%%%%%%%%%%%%%%%%%%%%%%%%%%%%%%%%%%%%%%%%%%%%%%%%%%%%%%%%%
\subsection{Bandits with concave rewards and convex knapsacks (\BwCR)}
\label{prelim:BwC}
\label{prelim:BwR}
\label{prelim:BwCR}
In this paper we consider a substantial generalization of \BwK, to include arbitrary concave rewards and arbitrary convex constraints. This is essentially the most general convex optimization problem. We consider the problem with only convex constraints (\BwC), and the problem with only concave rewards (\BwR) as special cases. 

%We generalize the linear knapsack constraints to convex constraints. This setting is motivated by Blackwell's approachability, which considers a two player vector valued game with biaffine payoffs. We discuss this connection towards the end of this subsection. 
In the Bandits with concave rewards and convex knapsacks (\BwCR) setting, on playing an arm $\tarm$ at time $t$, we observe a vector $\cv_t \in [0,1]^d$ generated independent of the previous observations, from a fixed but unknown distribution such that $\Ex[\cv_t | \tarm]=\mcm_{\tarm}$, where $\mcm \in [0,1]^{d\times m}$. We are given a convex set $S$, and a concave objective function $f:[0,1]^d \rightarrow [0,1]$. We further make the following assumption regarding Lipschitz continuity of $f$. 
\begin{assumption}
\label{assum:Lcont}
Assume that function $f$ is $L$-lipschitz with respect to norm $||\cdot||$, i.e., $f(\x) -f(\y) \le L||\x-\y||$. Since $f$ is concave, this is equivalent to the condition that for all $\x$ in the domain of $f$, and all supergradients $\bs{g} \in \partial f(\x)$, we have that $||\bs{g}||_* \le L$, where $||\cdot||_*$ is the dual norm (refer to Lemma 2.6 in \cite{Shalev-Shwartz12}). 
%norm of its subgradient is bounded by $L$, i.e.,
%\begin{eqnarray*}
%& f(\x) -f(\y) \le L||\x-\y||, & \forall \x,\y, \text{ and, } \\
%& ||\bs{g}_x||_* \le L, & \forall \x,\exists \bs{g}_x\in \partial f(\x).
%\end{eqnarray*}
%The two conditions are equivalent for differentiable functions. \textcolor{red}{Check, and add a reference?}
\end{assumption}

The goal is to make the average of the observed vectors $\frac{1}{T} \sum_t \cv_t$ be contained in the set $S$, and at the same time maximize $f(\frac{1}{T} \sum_t \cv_t)$. Let $\OPT_f$ denote the expected value of the optimal dynamic solution to this problem. Then, the following lemma provides a benchmark for defining regret. The proof follows simply from concavity of $f$, and is provided in Appendix \ref{app:prelims}.%\commentShipra{The following is new.}
\begin{lemma}
\label{lem:benchmark-BwCR}
There exists a distribution $\dis^* \in \Delta_m$, such that $\mcm \dis^* \in S$, and $f(\mcm \dis^*) \ge \OPT_f$.
\end{lemma}

We minimize two kinds of regret: regret in objective and regret in constraints. The (average) regret in objective is defined as
\LEQ{0in}{0in}{}{\areg_1(T) :=  \OPT_f -  f(\frac{1}{T} \sumT_{t=1}^T \rv_t) \le  f(\mcm\dis^*) - f(\frac{1}{T} \sumT_{t=1}^T \rv_t).}
And, (average) regret in constraints is the distance of average observed vector from $S$,
\LEQ{0in}{0in}{}{\areg_2(T) := d(\frac{1}{T}\sumT_{t=1}^T \cv_t, S),}
where the distance function $d(\x, S)$ is defined as $||\x-\proj_S(\x)||$, $\proj_S(\x)$ is the projection of $\x$ on $S$, and $||\cdot||$ denotes an $L_q$ norm.

\comment{
Now, due to concavity of $f$,
\begin{equation} 
\frac{\OPT}{T} = \Ex_{\optdis}[f(\rv)] \le  f(\mrm \optdis).
\end{equation}
Therefore,
\begin{equation}
 \areg_1(T) \le  f(\mrm \optdis)-  f(\frac{1}{T} \sum_t \rv_t)
\end{equation}
}
\commentShipra{For compactness of description, we will call an algorithm $\regB$-optimal, if its average regret in time $T$ can be bounded as
\LEQ{0in}{0in}{eq:regBoptimal}{\areg_1(T)\le \frac{L}{T} \cdot \regB, \text{~and,~} \areg_2(T)\le \frac{1}{T} \cdot \regB}
with probability $1-\delta$. 
}

Below, we describe some special cases and extensions of this setting.
\paragraph{Hard constraints}
%\label{rem:shrunkenSet}
In some applications, the constraints involved are hard constraints, that is, it is desirable that they are satisfied with high probability even if at a cost of higher regret in the objective. Therefore, we may want to tradeoff the regret in distance from $S$ for possibly more regret in objective $f$. While this may not be always doable, under following conditions a simple modification of our algorithm can achieve this: the set $S$ and function $f$ are such that it is easy to define and use a shrunken set $S^{\epsilon}$ for any $\epsilon \in [0,1]$, defined as a subset of $S$ such that points within a distance of $\epsilon$ from this set lie in $S$. And, $S^{\epsilon}$ contains at least one good point $\mcm \dis$ with objective function value within $K\epsilon$ of the optimal value. More precisely,
\begin{eqnarray}
\label{eq:shrunken-condition}
& d(\x,S^{\epsilon})\le \epsilon \Rightarrow \x\in S, & \text{ and} \nonumber\\
\exists \dis\in \Delta_m: & \mcm \dis \in S^{\epsilon}, f(\mrm \dis) \ge f(\mrm \dis^*) - K\epsilon, &
\end{eqnarray}
for some $K \ge 0$. A special case is when $S$ is a downward closed set, $f$ is linear, and distance is $L_{\infty}$ distance. In this case, we can define $S^{\epsilon} = \{\x(1-\epsilon), \forall \x\in S\}$, for which $\mrm \dis^* (1-\epsilon) \in S^{\epsilon}$ and $f(\mrm \dis^* (1-\epsilon)) \ge (1-\epsilon)f(\mrm \dis)$. 

In our algorithms, we will be able to simply substitute $S^{\epsilon}$ for $S$ to achieve the desired tradeoff. 
This observation will be useful for \BwK~ problem, which involves hard (downward closed) resource consumption constraints -- the algorithm needs to abort when the resource constraints are violated.
%%Note that simplex constraint =1 can be relaxed to <=1 if f is monotone or linear. Not clear in other cases

\paragraph{Linear contextual version of \BwCR} We also consider an extension of our techniques to the linear contextual version of the \BwCR~problem, which can be derived from the linear contextual bandits problem \cite{Auer2002, Chu2011}. In this setting, every arm $i$ and component $j$ is associated with a context vector $\cx_{ji}$, which is known to the algorithm. There is an unknown  $n$-dimensional weight vector $\wt_j$ for every component $j$, such that $\mcm_{ji}=\cx_{ji} \cdot \wt_j$. Note that effectively, the $d$ $n$-dimensional weight vectors are the unknown parameters to be learned in this problem, where $n$ could be much smaller than the number of arms $m$. Algorithms for contextual bandits are expected to take advantage of this structure of the problem to produce low regret guarantees even when the number of arms is large.

In a more general setting, the context vector for arm $i$ could even change with time (but are provided to the algorithm before taking the decision at time $t$), however that can be handled with only notational changes to our solution, and for simplicity of illustration, we will restrict to static contexts in the main body of this paper. 

\paragraph{\BwK, \BwR, and \BwC~as special cases} Observe that \BwCR~subsumes the \BwK~problem, on defining objective function $f(\x)=x_1$, and $S:=\{\x:\x_{-1} \le \frac{B}{T} {\bf 1}\}$. 
We define Bandits with concave Rewards (\BwR) as a special case of \BwCR~ when there are no constraints, i.e., the set $S=\mathbb{R}^n$. And, Bandits with Convex knapsacks (\BwC) as the special case when the goal is only to satisfy the constraints, i.e. there is no objective function $f$. The average regret for \BwR~ in time $T$ is $\areg_1(T)$, and for \BwC~ it is $\areg_2(T)$. 
\commentShipra{We will call an algorithm $\regB$-optimal if it provides an average regret bound of $\frac{L}{T}\regB$ for \BwR~and  $\frac{1}{T} \regB$ for \BwC.}
%%%%%%%%%%%%%%%%%%%%%%%%%%%%%%%%%%%%%%%%%%%%%%%%%%%%%%%%%%%%%%%%%%%%%%%
\subsection{Summary of Results}
\label{sec:results}

Our main result is that a natural extension of UCB algorithm (Algorithm \ref{algo:UCB-BwCR}) for \BwCR~ achieves bounds of
%is $\regB$-optimal with \EQ{0in}{0in}{\regB=\normOned \sqrt{m T \ln(\frac{mTd}{\delta})}.}
\EQ{0in}{0in}{O(L\normOned \sqrt{\frac m T \ln(\frac{mTd}{\delta})}), \text{ and } O(\normOned \sqrt{\frac m T \ln(\frac{mTd}{\delta})}),}
 with probability $1-\delta$, on the average regret in the objective ($\areg_1(T)$) and distance from constraint set ($\areg_2(T)$), respectively. 
Here $\normOned$ denotes the norm of $d$-dimensional vector of all $1$'s, with respect to the norm used in the Lipschitz condition of $f$, and, in defining the distance from set $S$, respectively.

We extend our results to the linear contextual version of \BwCR, and provide an algorithm 
%that is $\regB$-optimal with
%\EQ{0in}{0in}{\regB=\normOned n~\sqrt{T\ln(\frac{Td}{\delta})},}
with average regret bounds of
\EQ{0in}{0in}{O(Ln\normOned \sqrt{\frac{1}{T}\ln(\frac{Td}{\delta})}), \text{ and } O(n\normOned \sqrt{\frac{1}{T}\ln(\frac{Td}{\delta}))},}
respectively, 
when contexts are of dimension $n$. Note that these regret bounds do not depend on the number of arms $m$, which is crucial when number of arms is large, possibly infinite.

Note that \BwCR~subsumes the \MAB~ problem, and the contextual version of \BwCR~subsumes the linear contextual bandits problem, with $d=1, L=1$ and $S=\mathbb{R}^n$. And, our regret bounds for these problems match the lower bounds provided in \citet{BubeckC12} (Section 3.3) and \citet{DaniHK08}, respectively, within logarithmic factors.
A more refined problem-dependent lower bound (and matching upper bound) for the special case of \BwK~ was provided in \cite{BwK}. We show that our UCB algorithm when specialized to this case (Algorithm \ref{algo:UCB-BwK}) achieves a regret bound of
\EQ{0in}{0in}{\reg(T)=O\left(\sqrt{\log(\frac{mdT}{\delta})} (\OPT \sqrt{\frac{m}{B}} + \sqrt{m \OPT} + m \sqrt{\log(\frac{mTd}{\delta})})\right),}
 which matches the bounds of \cite{BwK}.
Thus, our UCB based algorithms provide near-optimal regret bounds. Precise statements of these results appear as Theorem \ref{th:UCB-BwCR} and Theorem \ref{th:UCB-BwK}.

Section \ref{sec:efficient} and \ref{sec:eff-BwCR} are devoted to developing a general framework for converting the UCB algorithm to fast algorithms. We provide algorithms \BwC~ and \BwR~ for which the arm selection problem at time $t$ is simply of the form: 
\EQ{0in}{0in}{\tarm = \arg \max_{i=1,\ldots, m} \omega_{t,i}.}
where $\omega_{t,i}$ for every $i$, can be computed using history until time $t-1$ in $O(d)$ time. 
These fast algorithms can be viewed as approximate primal and dual implementations of the UCB algorithm, and come with a cost of increased regret, but we show that the regret increases by only constant factors. The derivation of these fast algorithms from UCB also provides interesting insights into connections between this problem, the Blackwell approachability problem, and the Frank-Wolfe projection technique for convex optimization, which may be of independent interest.

%%%%%%%%%%%%%%%%%%%%%%%%%%%%%%%%%%%%%%%%%%%%%%%%%%%%%%%%%%%%%%%%%%%%%%%%%%%%%%%%%
\subsection{Related Work}
The \BwCR~problem, as defined in the previous section, is closely related to the stochastic multi-armed bandits (\MAB) problem, to the generalized secretary problems under stochastic assumption, and to the Blackwell approachability problem. As we mentioned in the introduction, the major difference between the classic \MAB~ model and settings like \BwCR~(or \BwK) is that the latter allow for ``global" constraints -- constraints on decisions accross time. The only global constraint allowed in the classic \MAB~ model is the time horizon $T$.  

Generalized secretary problems under i.i.d. distribution include online stochastic packing and covering  problems (e.g., \citep{Devanur2011}, \citep{Feldman10}). These problems involve ``global" packing or covering constraints on decisions over time, as we have in \BwCR. However, a major difference between the secretary problems and a bandit setting like \BwCR~ is that in secretary problems, {\it before} taking the decision at time $t$ the algorithm knows how much the reward or consumption (or in general $\cv_t$) will be for every possible decision. On the other hand, in the \BwCR~ setting, $\cv_t$ is revealed {\em
 after} the algorithm chooses the arm to play at time $t$. One of the ideas in this paper is to estimate the observations at time $t$ by \UCB~estimates computed using only history til time $t-1$, and before choosing the arm $\tarm$. This effectively reduces the problem to secretary problem, with error in the \UCB~estimates to account for in regret bounds.

%\subsection{The Blackwell approachability problem}
Blackwell approachability problem considers a two player vector-valued game with a bi-affine payoff function,
%  with $m$ rows and $|{\cal W}|$ columns. The payoff function $r:\Delta_m \times \Delta_{|{\cal W}|} \rightarrow \mathbb{R}^d$ is biaffine 
 $r(\dis,\q)=\dis^TM\q$.
%where $M_{i\omega}=r(e_i, e_\omega)$.  
%This assumption is required to preserve the expectation under mixed strategies,
%so that $\Ex_{p,q}[r(e_i,e_\omega)]=\dis^TM\q=r(\dis,\q).$
Further, it is assumed that for all $\q$, there exists a $\dis$ such that $r(\dis,\q) \in S$.
The row player's goal is to direct the payoff vector to some convex set $S$. %That is, minmize the distance  $d(\frac{1}{T}\sum_{t=1}^T r(\bs{e}_{i_t},\bs{e}_{\omega_t}), S)$.
The Bandit with convex knapsacks (\BwC) problem is closely related to the Blackwell approachability problem. 
The row player is the online algorithm and the column player is nature. 
However, in this case the nature always produces its outcome using a \emph{fixed} (but unknown) mixed strategy (distribution) $\q^*$. Also, this means a weaker assumption should suffice: there exists a $\dis^*$ for this particular $\q^*$, such that $r(\dis^*,\q^*) \in S$ (stated as the assumption $\exists \dis^*, \mcm \dis^*\in S$). 
The bigger difference algorithmically is that there is nothing to statistically estimate in the Blackwell approachability problem, 
the only unknown is the column player strategy which  may change every time. 
On the other hand, esitmating the expected consumption is inherently the core part of any algorithm for \BwC. 

%(Technically, for \BwC, one could define a column for each possibility, 
%blowing up the number of columns, but such a reduction is not very useful.) 
Due to these differences, algorithms for none of these related problems directly solve the \BwCR~problem. Nonetheless, the similarities suffice to inspire many of the ideas for computationally efficient algorithms that we present in this paper. 

The work closest to our work is that of \citet{BwK} on the \BwK~problem. We successfully generalize their setting to include arbitrary convex constraints and concave objectives, as well as linear contexts. Additionally, we demonstrate that a simple and natural extension of UCB algorithm suffices to obtain optimal regret for \BwCR~which subsumes \BwK, and provide generalized techniques for deriving multiple efficient implementations of this algorithm -- one of which reduces to an algorithm similar to the PD-BwK algorithm of \citet{BwK} for the speical case of \BwK. 
%As we will see later in the text, combining techniques from algorithms for Blackwell approachability problem with the UCB techniques for the multi-armed bandits problem will lead to more efficient algorithms for \BwCR problem.
 
\comment{
On playing arm $\tarm$ at time $t$, the algorithm observes vector $\cv_t=r(e_{\tarm},e_{\omega_t})$, where $\omega_t$ is outcome produced by nature. 
%Now, interpret $\cv_t := r(e_{i_t},e_{\omega_t})$ as observation vector at time $t$, and for every $i$, mean $\mcm_i :=r(e_{i}, q^*)$. 
Then, mean of the observed vector $\Ex_{q^*}[\cv_t | \tarm] = r(e_{i}, q^*) =: \mcm_{\tarm}$, due to biaffinity of $r(\cdot, \cdot)$. Also, $r(p^*,q^*)=\sum_i p^*_i r(e_i,q^*)=\sum_i p^*_i\mcm_i$, so the condition $r(p^*,q^*) \in S$ is equivalent to saying that $\mcm \dis^* \in S$.

%The goal of the online algorithm is to pick an arm $\tarm$ at time $t$ using only information till time $t-1$ and minimize the distance of mean reward vector from set $S$, i.e. minimize
%$d(\frac{1}{T}\sum_{t=1}^T r(e_{i_t},e_{\omega_t}), S).$
%where $i_t$ denotes the arm played by algorithm at time $t$, and $\omega_t$ denotes the outcome picked by nature at time $t$, i.i.d. from distribution $$.
 And, the regret
 $$\areg_2(T)=  d(\frac{1}{T}\sum_{t=1}^T \cv_t, S) = d(\frac{1}{T}\sum_{t=1}^T r(e_{\tarm},e_{\omega_t}), S).$$
}

%========================================================================

\subsection{Fenchel duality} 
\label{sec:Fenchel} 
Fenchel duality will be used throughout the paper, below we provide some background on this useful mathematical concept. We define the Fenchel conjugate of $f$ as 
\EQ{0in}{0in}{f^*(\thetaV):=\max_{\y \in [0,1]^d} \{ \y \cdot \thetaV + f(\y)\}}

Suppose that $f$ is a concave function defined  on $[0,1]^d$, and as in Assumption \ref{assum:Lcont},
at every point $\x$, every supergradient $\bs{g}_{x}$ of $f$ has bounded dual norm $||\bs{g}_x||_* \le L$. 
% set $C$, ($K_t \subseteq C$). This could be a simple set, e.g. $C={\mathbb R}^r$.
%The domain of $f^*$ is  $||\theta|| \leq L$. 
Then, the following dual relationship is known between $f$ and $f^*$. A proof is provided in Appendix \ref{app:Fenchel} for completeness. 
%------------------------------------------------------
\begin{lemma}
\label{lem:FenchelDuality} $f(\z) = \min_{||\thetaV||_* \le L} f^*(\thetaV)-\thetaV \cdot \z.$
%where $L$ is the Lipschitz constant of $f$.
\end{lemma}
%----------------------------------------------------------

A special case is when $f(\x) = - d(\x,S)$ for some convex set $S$. This function is $1$-Lipschitz with respect to norm $||\cdot||$ used in the definition of distance. In this case, $f^*(\thetaV) = h_S(\thetaV):=\max_{\y\in S} \thetaV\cdot \y$, and Lemma \ref{lem:FenchelDuality} specializes to 
%\begin{equation} 
%\label{eq:distnhs} 
\EQ{0in}{0in}{
d(\x,S) = \max_{||\thetaV||_* \le 1}\thetaV \cdot \x - h_S(\thetaV).}
%\end{equation} 
The derivation of this equality also appears in \citet{blackwell2011}.

%======================================================================================

\subsection{Notations}
We use bold alphabets or bold greek letters for vectors, and bold capital letters for matrices. Most matrices used in this paper will be $d\times m$ dimensional, and for a matrix $\bs{A}$, $A_{ji}$ denotes its $ji^{th}$ element, $\bs{A}_i$ denotes its $i^{th}$ column vector, and $\bs{A}_j$ its $j^{th}$ row vector. For matrices which represent  time dependent estimates, we use $\bs{A}_t$ for the matrix at time $t$, and $\bs{A}_{t,i}, \bs{A}_{t,j}$ and ${A}_{t,ji}$ for its $i^{th}$ column, $j^{th}$ row, and $ji$ component, respectively. 
For two vectors $\x, \y$, $\x\cdot\y$ denotes the inner product.

%% file: applications.tex
\section{Applications}
\label{sec:applications}
Below, we demonstrate that \BwCR~ setting and its extension to contextual bandits allows us to effectively handle much richer and complex models in applications like sensor networks, crowdsourcing, pay-per-click advertising etc., than those permitted by multi-armed bandits (\MAB), or bandits with knapsacks (\BwK) formulations. While some of these simply cannot be formulated in the \MAB~ or \BwK~ frameworks, others would require an exponential blowup of dimensions to convert the convex constraints to linear knapsack or covering constraints. 

\paragraph{Sensor networks, network routing}

Consider a sensor network with $m$ sensors, each sensor $i$ covering a subset $A_i$ of $N$ points, where $N>>m$, and $N$ could even be exponential compared to $m$. Taking a reading from any sensor costs energy. Also, a sensor measurement may fail with probability $q_i$. The aim is to take atmost $T$ measurements such that each point has at least $b$ successful readings. We are given that there exists a strategy for selecting the sensors, so that in expectation these covering constraints can be satisfied. A strategy corresponds to a distribution $\dis \in \Delta_m$ such that you measure sensor $i$ with probability $p_i$. We are given that
\begin{center} $\exists \dis^* \in \Delta_m, T\sum_{i: k\in A_i} p^*_i q_i\ge b, \forall k=1,\ldots, N. $\end{center}

We can model this as \BwC~by having $\cv_t \in \{0,1\}^m$ (i.e., $d=m$), where on playing arm $\tarm$, $\cv_{t,\tarm}$ denotes whether the sensor $\tarm$ was successfully measured or not: $\cv_{t,\tarm}={\bs e}_{\tarm}$ with probability $q_{\tarm}$, and $\bs{0}$ otherwise, and $\Ex[\cv_{t} |\tarm]=\mcm_{\tarm}$ where $\mcm$ is am $m\times m$ diagonal matrix with $\mcm_{ii}=q_i$. Define $S$ as 
\begin{center}$S=\{\x \in [0,1]^m : \sum_{i: k\in A_i} x_i \ge \frac{b}{T}, k=1,\ldots, N\}.$\end{center} 
Note that $S$ is an $m$-dimensional convex set. Then, we wish to achieve $\frac{1}{T} \sum_{t=1}^T \cv_t \in S$. And, from above there exists $\dis^* \in \Delta_m$ such that $\mcm \dis^* \in S$. Our algorithms we will obtain $O(||\bs{1}_m||\sqrt{\frac{m}{T} \log(\frac{mT}{\delta}}))$ regret as per the results metioned above ($d=m$).

Note that if we try to frame this problem in terms of linear covering constraints, we need to make $\cv_t$ to be $N$ dimensional (i.e, $d:=N$), where $N>>m$. On playing arm $i$, $\cv_{t}=\bs{e}_{A_i}$ with probability $p_{i}$. Then, the constraints can be written as linear constraints $ \sum_{t} \cv_{t,j} \ge b, j=1,\ldots, N.$
However, in that case, $d=N$ will result in a $||\bs{1}_N|| \sqrt{\log(N)}$ term in the regret bound, which can be exponentially worse than $||\bs{1}_m|| \sqrt{\log(m)}$.

Similar applications include {\em{crowdsourcing}} a survey or data collection task, where workers are sensors each covering his/her (overlapping) neighborhood, and {\em{network monitoring}}, where monitors located at some nodes of the network are sensors, each covering a subset of the entire network. 

Another similar application is {\em network routing}, where routing requests are arriving online. There is a small number ($d$) of request types, and the hidden parameters to learn are expected usage for each type of request. But, there is a capacity constraint on each of the $N >>d$ edges. Then, modeling it as \BwK~would get an $||\bs{1}_N|| \sqrt{\log(N)}$ term in the regret bound instead of $||\bs{1}_d|| \sqrt{\log(d)}$.

\paragraph{Pay-per-click advertising}
Pay-per-click advertising is one of the most touted applications for \MAB, where explore-exploit tradeoff is observed in ad click-through rate (CTR) predictions. 
%However, richer models are required to effectively formulate this problem in practice. 
Our \BwCR~ formulation with its contextual extension can considerably enrich the \MAB~ formulations of this problem. 
Contexts are considered central to the effective use of bandit techniques in this problem -- the CTR for an ad impression depends on the (query, ad) combination, and there are millions of these combinations, thus millions of arms. Contextual setting allows a compact representation of these arms as $n$-dimensional feature (context) vectors, and aims at learning the best weight vector that maps features to CTR.  

\BwCR~ allows using the contextual setting along with multiple complex constraints on the decision process over time. In addition to simple budget constraints for every advertiser/campaign, we can efficiently represent budget constraints on family of overlapping subset of those, without blowing up the dimension $d$, as explained in some of our earlier applications. 

The ability to maximize a concave reward function is also very useful for such applications. 
Although in most models of pay-per-click advertising the reward is some simple linear function, the reality is more complex. 
A typical consideration is that advertisers (in dislay advertising) desire a certain mixture of different demographics  such as equal number of men and women, or equal number of clicks from different cities. These are not hard constraints -- the closer to the ideal mixture, the better it is. This is naturally modeled as a concave reward function of the vector of the number of clicks of each type the advertiser recieves. 

%\paragraph{Risk-aversion} 
%Most current models assume risk-neutrality in the objective function, i.e., that the goal is to maximizes the expected reward. 
%It is conceivable that a big variance in the reward is undesirable in many applications and a risk adjusted function is more appropriate. Such reward functions are naturally concave. 
Further, we can now admit more nuanced risk-sensitive constraints. This includes convex risk functions on budget expenditure or on distance from the target click or revenue performance.

%% file: ucb.tex
\section{UCB family of algorithms}
In this section, we present algorithms derived from the UCB family of algorithms \cite{Auer2002} for the multi-armed bandit problems. We demonstrate that simple extensions of UCB algorithm provide near-optimal regret bounds for \BwCR~ and all its extensions introduced earlier. In particular, our UCB algorithm will match the optimal regret bound provided by \citet{BwK} for the special case of \BwK.

We start with some background on the UCB algorithm for classic multi-armed bandit problem.
In the classic multi-armed bandit problem there are $m$ arms and on playing an arm $\tarm$ at time $t$, a reward $\rs_t$ is generated i.i.d. with fixed but unknown mean $\mrvS_{i_t}$. The objective is to choose arms in an online manner in order to minimize regret defined as $\sum_{t=1}^T (\mrvS_{i^*} - \rs_t) $, where $i^* = \arg\max_i \mrvS_i$. 

\UCB~algorithm for multi-armed bandits was introduced in \citet{Auer2002}. The basic idea behind this family of algorithms is to use the observations from the past plays of each arm $i$ at time $t$ to construct estimates ($\UCB_{t,i}$) for the mean reward $\mrvS_i$. 
These estimates are constructed to satisfy the following key properties.
\begin{enumerate}
\item The estimate $\UCB_{t,i}$ for every arm is guaranteed to be larger than its mean reward with high probability, i.e., it is an \emph{Upper Confidence Bound} on the mean reward.
\EQ{0in}{0in}{\UCB_{t,i} \ge \mu_i, \forall i, t}
\item As an arm is played more and more, its estimate should approach the actual mean reward, so that with high probability, the total difference between estimated and actual reward for the \emph{played arms} can be bounded as
\EQ{0in}{0in}{|\sumT_{t=1}^T (\UCB_{t,\tarm}-\rs_t)| \le \tilde{O}(\sqrt{mT}).}
This holds irrespective of how the arm $\tarm$ is chosen.
\end{enumerate}
At time $t$, the UCB algorithm simply plays the best arm according to the current estimates, i.e., the arm with the highest value of $\UCB_{t,i}$.
\EQ{0in}{0in}{\tarm = \arg \max_i \UCB_{t,i}.}
Then, a corollary of the first property above, and the choice of arm made by algorithm, is that with high probability,
\EQ{-0.1in}{0in}{\UCB_{t,\tarm} \ge \mu_{i^*}.}
Using above observations, it is straightforward to bound the regret of this algorithm in time $T$.
\EQ{-0.1in}{0in}{\reg(T) = \sum_{t=1}^T (\mrvS_{i^*} - \rs_t) \le \sum_{t=1}^T (\UCB_{t,\tarm} - \rs_t) \le \tilde{O}(\sqrt{mT}).}

%$$\UCB_{t,i} := \min\{1,\hat{\mu}_{t,i} + 2 \rad(\hat{\mu}_{t,i}, k_{t,i}+1)\}, \ \ i=1,\ldots, m$$
%where $\rad(\nu, N) := \sqrt{\frac{\crad \ \nu}{N}} + \frac{\crad}{N}$, $k_{t,i}$ is the number of plays of arm $i$ before time $t$ and $\hat{\mu}_{t,i}$ is the empirical mean reward for those plays. 
%Standard large deviation bounds are used to guarantee that for each arm $i$, its upper confidence bound is larger than its mean reward $\mrvS_i$ with high probability. Then, the algorithm pretends that the upper confidence bound $\UCB_{t,i}$ is the mean reward for arm $i$ and plays the best arm accordingly. That is, it plays arm $\tarm = \arg \max \UCB_{t,i}$.
%The regret at time step $t$ is then bounded by the difference of $\UCB_{t,i}$ and actual mean $\mrvS_i$ for the played arm. 

In our UCB based algorithms, we use this same basic idea for algorithm design and regret analysis. 
%==================================================================================================
\subsection{Bandits with concave rewards and convex knapsacks (\BwCR)}
Since, our observation vector cannot be interpreted as cost or reward, we construct both lower and upper confidence bounds, and consider the range of estimates defined by these. More precisely, for every arm $i$ and component $j$, we construct two estimates $\LCB_{t,ji}(\mcm)$ and $\UCB_{t,ji}(\mcm)$ at time $t$, using the past observations. The estimates for each component are constructed in a manner similar to the estimates used in the UCB algorithm for classic MAB, and satisfy the following generalization of the properties mentioned above. 

\begin{enumerate}
\item The mean for every arm $i$ and component $j$ is guaranteed to lie in the range defined by its estimates $\LCB_{t,ji}(\mcm)$ and $\UCB_{t,ji}(\mcm)$ with high probability. That is, 
\LEQ{0in}{-0.1in}{eq:prop1}{\mcm \in \HC_t, \text{ where},}
\LEQ{0in}{0in}{def:HC}{\cHypercube_t:=\{\mcme: \mcmeS_{ji} \in [\LCB_{t,ji}(\mcm), \UCB_{t,ji}(\mcm)],  j=1,\ldots, d, i=1,\ldots, m\}.}
\item Let arm $i$ is played with probability $p_{t,i}$ at time $t$. Then, the total difference between estimated and actual observations for the \emph{played arms} can be bounded as
%Then, as an arm is played its estimate for all components should approach the actual mean, so that with high probability, the total difference between estimated and actual observations for the \emph{played arms} can be bounded as
\LEQ{0in}{0in}{eq:prop2}{||\sumT_{t=1}^T (\mcme_t \dis_t - \cv_t)|| \le \diff(T),}
for any $\{\mcme_t\}_{t=1}^T$ such that $\mcme_t \in \HC_t$. Here, $\diff(T)$ is typically $\tilde{O}(\normOned \sqrt{mT})$. 
\end{enumerate}
A direct generalization of Property (2) from the MAB analysis mentioned before would have been a bound on $||\sumT_{t=1}^T (\mcme_{t,\tarm} - \cv_t)||$. However, since we will choose a distribution $\dis_t$ over arms at time $t$ and sample $\tarm$ from this distribution, the form of bound in \eqref{eq:prop2} is more useful, and a straightforward extension. 
A specialized expression for $\diff(T)$ in terms of problem specific parameters will be obtained in the specific case of \BwK. 
As before, these are purely properties of the constructed estimates, and hold irrespective of how the choice of $\dis_t$ is made by an algorithm.

At time $t$, our UCB algorithm plays the best arm (or, best distribution over arms) according to the best estimates in set $\HC_t$.

%the range defined by $\LCB_t$ and $\UCB_t$.
\begin{algorithm}[H] 
\caption{UCB Algorithm for \BwCR}
\label{algo:UCB-BwCR}
  \begin{algorithmic}
\FORALL{$t=1, 2,\ldots, T$ \vspace{-0.1in}}
		\STATE \begin{equation}
\label{eq:algoChoice}
\dis_t = \begin{array}{rl}
\arg \displaystyle \max_{\dis\in \Delta_m} & \displaystyle\max_{\mrme \in \HC_t} f(\mrme \dis)  \\
\textrm{s.t.} & \min_{\mcme \in \HC_t} d(\mcme \dis, S)\le 0 
\end{array}
\end{equation}
If no feasible solution is found to the above problem, set $\dis_t$ arbitrarily.  
		\STATE Play arm $i$ with probability $p_{t,i}$.
\ENDFOR
%  	}
  	\end{algorithmic}
\end{algorithm}

Observe that when $f(\cdot)$ is a monotone non-decreasing function as in the classic MAB problem (where $f(x)=x$), the inner maximizer in objective of \eqref{eq:algoChoice} will be simply $\mrme_t = \UCB_t(\mrm)$, and therefore, for classic MAB problem this algorithm reduces to the UCB algorithm.

Let $\mrme_t, \mcme_t$ denote the inner maximizer and the inner minimizer in the problem \eqref{eq:algoChoice}. Then, a corollary of the first property above (refer to Equation \eqref{eq:prop1}) is that with high probability, \vspace{-0.1in}
\begin{equation}
\label{eq:cor}
f(\mrme_t\dis_t) \ge f(\mrm \dis^*), \ \ \ \mcme_t \dis_t \in S.
\end{equation}
This is because the conditions $\mrm \in \HC_t$ and $\mcm \dis^* \in S$ imply that $(\dis, \mcme, \mrme)=(\dis^*,\mcm, \mrm)$ forms a feasible solution for problem \eqref{eq:algoChoice} at time $t$. 

Using these observations, it is easy to bound the regret of this algorithm in time $T$.  With high probability,
\begin{equation}
\label{eq:regCalc}
\begin{array}{c}
\areg_1(T) \le  f(\mrm \dis^*) - f(\frac{1}{T}\sumT_{t=1}^T \rv_t) \le f(\mrme_t\dis_t) - f(\frac{1}{T}\sumT_{t=1}^T \rv_t) \le  \frac{L}{T} \diff(T),\vspace{0.1in}\\
\areg_2(T) = d(\frac{1}{T}\sum_{t=1}^T \cv_t, S) \le d(\frac{1}{T}\sumT_{t=1}^T \cv_t, \frac{1}{T}\sumT_{t=1}^T \mcme_t \dis_t) \le \frac{1}{T} \diff(T),
\end{array}
\end{equation}
where $\diff(T)=\tilde{O}(\normOned \sqrt{mT})$. Below is a precise statement for the regret bound.
\begin{theorem}
\label{th:UCB-BwCR}
%For the \BwCR~ problem, Algorithm \ref{algo:UCB-BwCR} is $\regB$-optimal, with $\regB=O( \normOned \sqrt{mT\log(\frac{mTd}{\delta})})$.
With probability $1-\delta$, the regret of Algorithm \ref{algo:UCB-BwCR} is bounded as
\EQ{0in}{0in}{\areg_1(T) = O(L \normOned \textstyle{\sqrt{\frac{\crad m}{T}} )}, \ \ \areg_2(T) = O(\normOned \textstyle{\sqrt{\frac{\crad m}{T}} )}}
where $\crad = O(\log(\frac{mTd}{\delta}))$, ${\bf 1}_d$ is the $d$ dimensional vector of all $1$'s.
\end{theorem}
The detailed proof with exact expressions for $\UCB_t(\mcm), \LCB_t(\mcm)$ is in Appendix \ref{app:UCB-BwCR}.
%--------------------------EXTENSIONS------------
\subsubsection{Extensions}
\label{rem:shrunkenSet}
\label{rem:contextual}
\paragraph{Linear Contextual Bandits.} It is straightforward to extend Algorithm \ref{algo:UCB-BwCR} to linear contextual bandits, using existing work on UCB family of algorithms for this problem. 
%the following high probability regret guarantees. 
%\begin{equation}
%\label{eq:reg-contextual}
%\areg_1(T) = O(L \normOned \textstyle{n\sqrt{\frac{\ln(dT/\delta)}{T}} )}, \ \ \areg_2(T) = O(\normOned \textstyle{n\sqrt{\frac{\log(dT/\delta)}{T}} )}.
%\end{equation}
Using techniques in \citet{oful, Auer2003}, instead of the hypercube $\HC_t$ at time $t$, one can obtain an ellipsoid such that the weight vector $\wt_j$ is guaranteed to lie in this elliposid, for every component $j$.
Then, simply substituting $\HC_t$ with these ellipsoids in Algorithm \ref{algo:UCB-BwCR} will provide an algorithm for the linear contextual version of \BwCR~with regret bounds
\EQ{0in}{0in}{\areg_1(T) = O(Ln\normOned \textstyle{\sqrt{\frac{\crad}{T}} )}, \ \ \areg_2(T) = O(n\normOned \textstyle{\sqrt{\frac{\crad}{T}} )},}
%$\regB$-optimal algorithm for linear contextual version of \BwCR~ with 
%$\regB=\normOned n\sqrt{T\ln(\frac{dT}{\delta})}$. 
with probability $1-\delta$. Here  $\crad = O(\log(\frac{mTd}{\delta}))$. Further details are in Appendix \ref{app:linUCB}.
%This can be very useful when $n<<m$, i.e. the underlying weight vector to be learned has much smaller dimension than the number of choices. In particular, this allows for the possibility of infinite number of arms defined by all context vectors in an $n$-dimensional continuous set (such as a simplex, or a unit ball). 

%\begin{theorem}
%\label{th:context-BwCR}
%For the \BwCR problem with linear contexts, with probability $1-\delta$, the regret of Algorithm \ref{algo:UCB-BwCR} is bounded as
%$$\areg_1(T) = O(L \normOned \textstyle{\sqrt{\frac{\crad n}{T}} )}, \ \ \areg_2(T) = O(\normOned \textstyle{\sqrt{\frac{\crad n}{T}} )}$$
%where $\crad = O(\log(\frac{mTd}{\delta}))$, ${\bf 1}_d$ is the $d$ dimensional vector of all $1$'s.
%\end{theorem}
%The case of contexts changing with time is not much different and we will obtain similar algorithm  ($\cx_{j\tarm}$ replaced by $\cx_{j\tarm}(t)$ at time $t$) and regret guarantees.
%\commentShipra{Should we state above regret as theorem?}

\paragraph{Hard constraints}
In this case, a shrunket set $S^{\epsilon}$ can be used instead of $S$ in Algorithm \ref{algo:UCB-BwCR} (refer to Section \ref{prelim:BwCR} for definition of $S^{\epsilon}$), with $\epsilon$ set to be an upper bound on $\areg_2(T)$. For example, $\epsilon$ can be set as  $\normOned\sqrt{\frac{\crad m}{T}}$ using results in Theorem \ref{th:UCB-BwCR}. Then, at the end of time horizon, with probability $1-\delta$, the algorithm will satisfy,
\EQ{0in}{0in}{d(\frac{1}{T}\sum_t \cv_t, S^{\epsilon}) \le \epsilon \Rightarrow \frac{1}{T}\sum_t \cv_t \in S,~ \text{and~} \areg_1(T) = O(L \normOned \textstyle{\sqrt{\frac{\crad m}{T}} + K \epsilon)}.}
%\textcolor{red}{check the definition of Sepsilon and proof of this remark}
%with probability $1-\delta$. 
%In particular, using $\epsilon=\normOned \textstyle{\sqrt{\frac{\crad m}{T}}}$, would ensure that $\frac{1}{T}\sumT_{t=1}^T \cv_t$ is in $S$ with high probability, without changing the order of regret in the objective.

%\end{remark}

%--------------------------------------------------------

%Next, we apply the algorithm and analysis of this section to the special case of Bandits with knapsacks problem (\BwK), and obtain near-optimal regret bounds that match those provided by \cite{BwK} for this problem.

%%%%%%%%%%%%%%%%%%%%%%%%%%%%%%%%%%%%%%%%%%%%

\subsection{Bandit with knapsacks (\BwK)}
This is a special case of \BwCR~ with $\rv_t=\{\rs_t; \Bcv_t\}$, $f(\x)=x_1$, and $S=\{\x:\x_{-1} \le \frac{B}{T}{\bf 1}\}$. 
%Since this problem requires hard constraints on the resource consumption, we use the version of algorithm with shrunken set $S^{\epsilon}$ instead of $S$ to tradeoff the regret in constraints for increased regret in the rewards (refer to Remark \ref{rem:tradeoff}). Here, $S^{\epsilon}$ is easy to define for $L_{\infty}$ distance, $S^{\epsilon}=\{\Bcv:\Bcv \le (\frac{B}{T}-\epsilon){\bf 1}\}$.
Then, the problem \eqref{eq:algoChoice} in Algorithm \ref{algo:UCB-BwCR} reduces to the following $\LP$.
\begin{equation}
%\label{def:LP}
\begin{array}{lcl}
\max_{\dis \in \Delta_m} & \UCB_t(\mrv) \cdot \dis & \\
{\rm s.t.} & \LCB_t(\Bmcm) \dis \preceq \frac{B}{T}{\bf 1}, &
\end{array}
\end{equation}
where $\UCB_t(\mrv) \in [0,1]^m$ denotes the UCB estimate constructed for $\mrv$ and $\LCB_t(\Bmcm) \in [0,1]^{d\times m}$ denotes the LCB estimate for $\Bmcm$.
Above is same as $\LP(\UCB_t(\mrv), \LCB_t(\Bmcm))$ (refer to Equation \eqref{def:LP}). 

Since this problem requires hard constraints on resource consumption, we would like to tradeoff the regret in constraint satisfaction for some more regret in reward. As discussed in Section \ref{rem:shrunkenSet}, one way to achieve this is to use a shrunken constraint set. For any $\mrv, \Bmcm$, we define $\LP(\mrv, \Bmcm, \epsilon)$ by tightneing the constraints in $\LP(\mrv, \Bmcm)$ by a $1-\epsilon$ factor, i.e. replacing $B$ by $(1-\epsilon)B$. Then, at time $t$, the algorithms simply solves $\LP(\UCB_t(\mrv), \LCB_t(\Bmcm), \epsilon)$ instead of $\LP(\UCB_t(\mrv), \LCB_t(\Bmcm))$.
%\commentShipra{TODO: refer to the remark on shrunken set.}

%Thus, at every time $t$, the algorithm computes the (near) optimal distribution pretending that the actual mean reward and mean costs were given by estimates $\UCB_t(\mrv)$ and $\LCB_t(\Bmcm)$, respectively. 

\begin{algorithm}[H] 
\label{algo:UCB-BwK}
\caption{UCB algorithm for \BwK}
  \begin{algorithmic}
\FORALL{$t=1, 2,\ldots, T$} 
		\STATE Exit if any resource consumption is more than $B$.
		\STATE Solve $\LP(\UCB_t(\mrv), \LCB_t(\Bmcm), \epsilon)$, and let $\dis_t$ denote the solution for this linear program. 
		\STATE Play arm $i$ with probability $\disS_{t,i}$.
\ENDFOR
%  	}
  	\end{algorithmic}
\end{algorithm}

%=============================================
\begin{theorem}
\label{th:UCB-BwK}
For the \BwK~problem, with probability $1-\delta$, the regret of Algorithm \ref{algo:UCB-BwK} with $\epsilon=\epsilonVal$, $\crad=\cradVal$,
is bounded as
\begin{center} $\reg(T) = O\left(\sqrt{\log(\frac{mdT}{\delta})} (\OPT \sqrt{\frac{m}{B}} + \sqrt{m \OPT} + m \sqrt{\log(\frac{mTd}{\delta})})\right).$\end{center}
\end{theorem}
\begin{proof}
We use the same estimates for each component as in the previous section, to construct $\UCB_t(\mrv)$ and $\LCB_t(\Bmcm)$. We show that these UCB and LCB estimates satisfy the following more specialized versions of the properties given by Equation \eqref{eq:prop1} and \eqref{eq:prop2}. With probability $1-(mTd) e^{-\Omega(\crad)}$,
%\begin{enumerate}
%\itemsep-2em
%\item \begin{equation}
\begin{eqnarray}
\text{(1)~~} & \UCB_t(\mrv) \succeq \mrv, \LCB_t(\Bmcm) \preceq \Bmcm. & \label{eq:prop1-BwK} \\
& & \nonumber\\
%\end{equation}
%\item 
\text{(2)~~} & \begin{array}{rcl}
\sum_{t=1}^T (\UCB_{t}(\mrv) \cdot \dis_t- \rs_t)| & \le  & O(\sqrt{\crad m\left(\sumT_t \rs_t \right)} + \crad m), \vspace{0.1in}\\
|\sum_{t=1}^T (\LCB_{t}(\Bmcm) \dis_t - \Bcv_t)| & \preceq & \epsilon B {\bf 1}. 
\end{array}& \label{eq:prop2-BwK}
\end{eqnarray}
%\end{enumerate}
Proof of the second property is similar to Lemma 7.4 of \cite{BwK}, and is provided in Appendix \ref{app:UCB-BwK} for completeness.

Then, similar to \eqref{eq:cor}, following is a corollary of the first property and the choice made by the algorithm at time step $t$.
\begin{equation}
\label{eq:algoCBexp}
\begin{array}{rcl}
\sum_{t=1}^T \UCB_t(\mrv) \cdot \dis_t & = & \LP(\UCB_t(\mrv), \LCB_t(\Bmcm), \epsilon) \ge \LP(\mrv, \Bmcm, \epsilon) \ge (1-\epsilon) \OPT, \vspace{0.1in}\\
\sum_{t=1}^T \LCB_{t}(\Bmcm) \dis_t & \preceq & (1-\epsilon)B {\bf 1}.
\end{array}
\end{equation}

Then, using the second property above and \eqref{eq:algoCBexp}, $\sum_{t=1}^T \Bcv_t \le B{\bf 1}$, and the algorithm will not terminate before time $T$. This means that the total reward for the algorithm will be given by $\ALGO = \sum_{t=1}^{T} \rs_t$.
Also, using the second property,
%Then, from \eqref{eq:algoCBexp} and \eqref{eq:algoDiff},
\begin{center}$\ALGO =  \sum_{t=1}^{T} \rs_t \ge (1-\epsilon)\OPT - O(\sqrt{\crad m \ \ALGO}) - O(\crad m )$\end{center}
%\\
%& \ge & (1-\epsilon)OPT -\sqrt{\crad (1-\epsilon) OPT} -O(\sqrt{\crad m \ ALGO}) - O(\crad m \log(T))\\
%& \ge & (1-\epsilon)OPT -\sqrt{\crad OPT} - O(\sqrt{\crad m \ ALGO})	 -O(\crad m \log(T))\\

Therefore, either $\ALGO \ge \OPT$ or
\begin{center}$\ALGO \ge (1-\epsilon)\OPT	 - O(\sqrt{\crad m \OPT})	 - O(\crad m ).$\end{center}
%That is, 
%\begin{eqnarray*}
%\reg(T) & \le & 2\epsilon \OPT + O(\sqrt{\crad m \OPT})	+ O(\crad m \log(T)).
%\end{eqnarray*}
Now, assuming $m \crad \le O(B)$, \footnote{This assumption was also made in \cite{BwK}} $\epsilon \OPT = O(\OPT\sqrt{\frac{m\crad}{B}})$.
Therefore,
\begin{center}$\reg(T) = \OPT-\ALGO \le O\left(\OPT \sqrt{\frac{\crad m}{B}} + \sqrt{\crad m \OPT}	+ \crad m \right).$\end{center}
Then, substituting $\crad=\Theta(\log(\frac{mTd}{\delta}))$, we get the desired result.
\end{proof}
%The previous lemma provided results for the distribution $\dis_t$ of arms played by the algorithm. The next lemma shows that a result similar to previous lemma holds with high probability, for the arms actually played by the algorithm.

%%%%%%%%%%%%%%%%%%%%%%%%%%%%%%%%%%%%%%%%%%%
\subsection{Implementability}
Next, we investigate whether our UCB algorithm is efficiently implementable. For the special case of \BwK~ problem, this reduces to Algorithm \ref{algo:UCB-BwK} which only requires solving an LP at every step. However, the poynomial-time implementability of Algorithm \ref{algo:UCB-BwCR} is not so obvious. Below, we prove that the problem \eqref{eq:algoChoice} required to be solved in every time step $t$ is in fact a convex optimization problem, with separating hyperplanes computable in polynomial time. Thus, this problem can be solved by ellipsoid method, and every step of  Algorithm \ref{algo:UCB-BwK} can be implemented in polynomial time. %\newline\\
%In the following sections, we present computationally more efficient primal and dual versions of this algorithm, with slight loss in regret, although the regret will remain of the same order. 

\begin{lemma}\label{lem:ucbImplementability} 
The functions $\psi(\dis):=\max_{\mrme \in \HC_t} f(\mrme \dis)$, and $g(\dis) = \min_{\mcme \in \HC_t} d(\mcme\dis, S)$ are concave and convex functions respectively, and the subgradients for these functions at any given point can be computed in polynomial time using ellipsoid method for convex optimization.
\end{lemma}
The proof of above lemma is provided in Appendix \ref{app:UCB-BwCR}.

%Note that even though Algorithm \ref{algo:UCB-BwCR} was shown to be implementable in polynomial time, it required solving a convex optimization problem in $\dis$ (possibly using ellipsoid method), for which computing the separating hyperplane at any point itself required solving a convex optimization problem in $\thetaV$ (again, possibly using ellipsoid method). Therefore, this algorithm may not be very efficient to implement. In the following sections, we present efficient algorithms inspired by algorithms for Blackwell's approachability problem \cite{blackwell2011}, and by Frank-Wolf projection-free techniques \cite{frank-wolfe}.

%% file: efficient.tex
\section{Computationally Efficient Algorithms for \BwC~ and \BwR} 
\label{sec:efficient}
In the UCB algorithm for \BwCR, at every time step $t$, we need to solve the optimization problem (\ref{eq:algoChoice}). 
%Or, in other words,
%$$\dis_t=\arg \min_{p\in \Delta_m} \min_{\mcme\in H_t} d(p\mcme, S).$$
%where $H_t$ is the hypercupe defined by the constraints $\mcme_{ij} \in [\LCB(t)_{ij}, \UCB(t)_{ij}]$.
%However, above problem is NP-hard (\textcolor{blue}{Requires proof}). 
Even though this can be done in polynomial time (Lemma \ref{lem:ucbImplementability}), this is an expensive step. 
It requires solving a convex optimization problem in $\dis$ (possibly using ellipsoid method), for which computing the separating hyperplane at any point itself requires solving a convex optimization problem (again, possibly using ellipsoid method).  For practical reasons, it is desirable to have a faster algorithm. In this section, we present alternate algorithms that are very efficient computationally at the expense of a slight increase in regret. The regret bounds remain the same in the $O(\cdot)$ notation and the 
increase is only in the constants. We present two such algorithms, a ``primal'' algorithm based on the Frank-Wolfe 
algorithm \citep{frank-wolfe} and a ``dual'' algorithm based on the reduction of the Blackwell approachability problem to online convex optimization (\OCO)~ in \citet{blackwell2011}. In this section, for simplicity of illustration, we consider only the \BwC~ and \BwR~ problems, i.e., the problem with only constraint set $S$, and the problem with only the objective function $f$, respectively.
In Section \ref{sec:eff-BwCR} we show that one could use any combination of these algorithms, or the UCB algorithm, for each of \BwC~and \BwR~to get an algorithm for \BwCR. 

The basic idea is to replace the convex optimization problem with its ``linearization'', which 
turns out to be a problem of optimizing a linear function over the unit simplex, and hence very easy to solve. 
For the \BwC~ problem, the convex optimization problem (\ref{eq:algoChoice}) specializes to finding a 
$\dis_t$ such that  $\mcme\dis_t \in S$ for some $\mcme \in \cHypercube_t$.  
In our ``linearized"' version, instead of this, we will only need to find a $\dis_t$ such that $\mcme\dis_t$ is in a halfspace containing the set $S$. 
A half space that contains $S$ and is tangential to $S$ is given by a vector $\thetaV$; 
such a halfspace is  $H_S(\thetaV) := \{ \x: \thetaV \cdot \x \leq h_S(\thetaV)\}$, 
where $h_S(\thetaV):=\max_{\s\in S} \thetaV\cdot \s$. 
Now given a $\thetaV_t$ in time step $t$, a point in $H_S(\thetaV_t)$ can be found by 
simply minimizing $\thetaV_t \cdot \x$, which is a linear function. This is exactly what the 
algorithm does, at each time step $t$, it picks a vector $\thetaV_t$ and sets 
\begin{equation} 
\label{eq:ptaslinearopt} 
(\dis_t, \mcme_t)=\arg \min_{\dis\in \Delta_m} \min_{\mcme \in H_t} \thetaV_t \cdot (\mcme\dis). 
\end{equation} 
% As was shown in Lemma \ref{lem:MinEst}, the 
The inner minimization is actually trivial and the optimal solution is at a vertex
of $\HC_t$, independent of the value of $\dis$, i.e., $\mcme_t = \vertex{\thetaV_t}$, where
%The following observation regarding linear minimization over set $\HC_t$ will be useful.
%\begin{lemma}
%\label{lem:MinEst}
%For any vector $\thetaV \in \mathbb{R}^d$, $\mcme^T\thetaV \in \mathbb{R}^m$ is component wise minimized over $\HC_t$ by a vertex of $\HC_t$, denoted as $\vertex{\thetaV}$, where
\begin{equation}
\label{eq:MinEst}
\vertex{\thetaV}_{ji} :=\left\{\begin{array}{ll}
\UCB_{t,ji}(\mrm), & \theta_j \le 0,\\
\LCB_{t,ji}(\mrm), & \theta_j > 0 
\end{array}
\right.,
\end{equation}
for $j=1,\ldots, d, i=1,\ldots, m$.
%\end{lemma}
%(See (\ref{eq:MinEst}) for the definition of $\vertex{\cdot}$.) 

%since the objective is  linear and the constraints on $\mcme$ are just  box constraints: the optimal solution, say  $\mcmeS_{t,ij}$,  is obtained by setting it to be $\LCB(t)_{ij}$ if $\thetaV_{t,j}\ge 0$ and $\UCB_{ij}(t)$ if $\theta_{t,j}<0$. 
With this observation, the outer minimization is also quite simple,  since it  optimizes a linear function over the unit  simplex 
and the optimal solution occurs at one of the vertices. It is solved by setting $\dis_t = {\bf e}_{i_t}$, where 
\begin{equation}
\label{eq:ptsimple}
\tarm=\arg \min_{i\in \{1,\ldots, m\}} \thetaV_t \cdot \mcme_{t,i}.
\end{equation}
Hence given $\thetaV_t$, the procedure for picking the arm $\tarm$ is quite simple. 

A generalization of this idea is used for \BwR: instead of optimizing $f$, we optimize a linear function that is tangential to $f$. 
A linear function that is tangential to $f$ at a point $\y$ (and is an upper bound on $f$ since $f$ is concave) is 
\[ l_f(\x;\y) := f(\y) + \nabla f(\y) \cdot (\x - \y) \geq f(\x) ~\forall \x,\y. \] 
Then, instead of maximizing $f(\mcme \dis)$ as in \eqref{eq:algoChoice}, we maximize $l_f(\mcme \dis;\y_t)$ over $\mcme \in \HC_t$ and $\dis\in \Delta_m$, for some $\y_t$. The latter is equivalent to minimizing $\x \cdot \thetaV_t$ where $\thetaV_t = -\nabla f(\y_t)$, therefore 
$\dis_t$ is still set as per (\ref{eq:ptaslinearopt}) (which reduces to the simple rule in \eqref{eq:ptsimple}). 

We introduce some notation here, let $\x_t:=\mcme_t\dis_t$,   $\optxv: = \mcm\optdis $, 
 $\avgt \x := \tfrac 1 T \sum_{s=1}^t\x_{s}$ and $\avgt \rv := \tfrac 1 T \sum_{s=1}^t\rv_{s}$.

The regret bound for the UCB algorithm followed rather straight-forwardly from the two properties (\ref{eq:prop1}) and (\ref{eq:prop2}), but the regret bounds for these algorithms will not be as easy.
For one, we no longer have (\ref{eq:cor}), instead we have the corresponding relations for $l_f$ and $H_S$ respectively: 
\begin{equation}
\label{eq:cor2}
\begin{array}{c}
l_f(\x_t;\y_t) \ge l_f(\optxv;\y_t) \geq f(\optxv) ,  \\
\x_t\in H_S(\thetaV_t).
\end{array}
\end{equation}
Since we don't have that $f(\x_t) \geq f(\optxv) $ (or $\x_t \in S$), the main task is to bound 
$f(\optxv)  - f(\avgT \x) $ (and $d(\avgT \x, S)$), and these will be extra terms in the regret bound. 
In particular, (\ref{eq:regCalc}) is replaced by 
\begin{equation}
\label{eq:regCalc2}
\begin{array}{rcl}
\areg_1(T) &&\le  f(\optxv) - f(\avgT  \rv) 
	\le f(\optxv)  - f(\avgT \x)+f(\avgT \x) - f(\avgT  \rv) \\
	&&\le f(\optxv)  - f(\avgT \x)+ \frac{L}{T} \diff(T).\vspace{0.1in}\\
\areg_2(T) &&= d(\avgT \rv,  S) \le d(\avgT \rv, \avgT \x) + d(\avgT \x, S)  \le \frac{1}{T} \diff(T)+ d(\avgT \x, S).
\end{array}
\end{equation}

The bounds on $f(\x^*) - f(\avgT{\x})$ and $d(\avgT{\x}, S)$ will depend on the choice of $\thetaV_t$s.
Each of the two algorithms we present provides a specific method for choosing $\thetaV_t$s to achieve desired regret bounds. 
%With this background, we now present the two algorithms. 

\input{dual}
%\input{blackwell} 
\input{frankwolfe}

%%%%%%%%%%%%%%%%%%%%%%%%%%%%%%%%%%%%%%%%%%%%%%%%%%%%%%%%%%%%%%%%%%%%%%%%%%%%%%%%%%%

%% file: dual.tex
\subsection{The dual algorithm}
\label{sec:dual}
This algorithm is inspired by the reduction of the Blackwell approachability problem to online convex optimization (\OCO) in \citet{blackwell2011}. It is also related to the fast algorithms to solve covering/packing LPs using multiplicative weight update \citep{Devanur2011} and the algorithm of \citet{BwK}. In fact, we give a {\em reduction}  to  \OCO; any algorithm for \OCO~can then be used.

In  \OCO, the algorithm has to pick a vector, say $\thetaV_t$ in each time step $t$. 
(The domain of $\thetaV_t$ is such that $||\thetaV_t||_*\le L$ for our purpose here, where $L$ is the Lipschitz constant of $f$, and $L=1$ for distance function.) 
Once $\thetaV_t$ is picked the algorithm observes a convex   {\em loss} function, $g_t$, and the process repeats. 
The objective is to minimize regret defined as
\EQ{0in}{0in}{\regOCO(T) :=  \sum_{t=1}^T g_t(\thetaV_t) - \min_{||\thetaV||_*\le L} \sum_{t=1}^T g_t(\thetaV).}
Recall from our discussion earlier, in each step $t$, the algorithm sets $\dis_t$ as per (\ref{eq:ptaslinearopt}) for some $\thetaV_t$. The choice of $\thetaV_t$ is via a reduction to \OCO: we define a convex  function $g_{t-1}$ based on 
the history upto time $t-1$ which is then fed as input to the \OCO~ algorithm, whose output 
$\thetaV_t$ is used in picking $\dis_t$. 
%\subsection{Bandits with concave rewards}
%To summarize the previous algorithm, in time step $t$, the algorithm picks $\theta_t$ and sets 
%$$(\dis_t,{\mcme_t}) =\arg \min_{\dis\in \Delta_m} \min_{\mcme \in H_t} \theta_t \cdot (\mcme\dis). $$
%$\theta_t$ is obtained as the output of an \OCO~ algorithm, whose input 
%$g_t$ was set as $g_t(\theta) =h_S(\theta)- \theta\cdot ({\mcme_t}\dis_t).$ 
%For \BwC, the only difference is in the defintion of $g_t$, now  
We first define $g_t$ for the \BwR~ problem; the $g_t$ for \BwC~ is obtained as a special case with $f(\x) = -d(\x,S)$. 
Define 
\EQ{0in}{0in}{g_t (\thetaV):= f^*(\thetaV) -\thetaV\cdot \x_t,}
where $f^*$ is the Fenchel conjugate of $f$,(see Section \ref{sec:Fenchel} for the definition), and $\x_t=\mcme_t\dis_t$. 

%The domain for the \OCO~ instance  is $||\theta||\leq L$. 
%
%Now, the algorithm will proceed as follows. We will use online convex minimization for choosing $\theta_t$ at time $t$, this choice will be explained shortly. Given a $\theta_t$, the algorithm will pick $\dis_t \in \Delta_m, \tilde{W}_t\in G_t$ that maximizes $f^*(\theta_t)-p\tilde{W}\theta_t$, or equivalently minimizes $p\tilde{W}\theta_t$. 

%We will use online convex minimization for choosing $\theta_t$ at time $t$. Define
%$$h_t(\theta) := f^*(\theta)-\theta^T\dis_t\tilde{W}_t$$
%Then, $\theta_{t+1}, ||\theta_{t+1}||\le L$ is chosen using online convex minimization update given $p_1, \tilde{W}_t, \ldots, \dis_t\tilde{W}_t$ in order to minimize total regret in time $t$,
%$${\cal R}^c(T)=\sum_t h_t(\theta_t) - \min_{||\theta||\le L} h_t(\theta).$$
%

\begin{algorithm}[H] 
\caption{Fenchel dual based algorithm for \BwR}
\label{algo:Dual}
  \begin{algorithmic}
	\STATE Inititalize $\thetaV_{1}$. 
\FORALL{$t=1, 2,\ldots, T$} 
		%\STATE Set $\tilde{W}_{ij}(t)=\UCB(W_{ij})$ for $\theta_j<0$, and $\tilde{W}_{ij}(t)=\LCB(W_{ij})$ for $\theta_j\ge 0$.
		\STATE Set $(\dis_t, {\mcme_t}) =\arg \min_{\dis\in \Delta_m, \mcme \in \cHypercube_t}  \thetaV_t \cdot (\mcme\dis)$.
		\STATE Play arm $i$ with probability $p_{t,i}$.
		\STATE Choose $\thetaV_{t+1}$ by doing an \OCO~update for the convex function 
		$g_t(\thetaV) =f^*(\thetaV) - \thetaV\cdot (\mcme_t \dis_t) .$
\ENDFOR
%  	}
  	\end{algorithmic}
\end{algorithm}

%The following Lemma generalizes (\ref{eq:distnhs}) and plays a similar role in the regret analysis. 
The following geometric intuition for the Fenchel conjugate is useful in the analysis: 
if $\y_t = \arg\max_{\y} \{ \y \cdot \thetaV_t + f(\y)\} $ 
then $-\thetaV_t \in \nabla f(\y_t)$ and 
$f^*(\thetaV_t) = \y_t \cdot \thetaV_t + f(\y_t) = l_f({\mathbf 0}; \y_t)$, i.e., $f^*(\thetaV_t)$ is the 
	$y$-intercept of $l_f(\x;\y_t)$. We can therefore rewrite $l_f( {\x} ; \y_t)$ in terms of $f^*$ as follows
\EQ{0in}{0in}{ l_f(\x;\y_t) = f^*(\thetaV_t)  - \thetaV_t\cdot \x . }
With this and (\ref{eq:cor2}), we have 
\EQ{0in}{0in}{g_t(\thetaV_t) = f^*(\thetaV_t) - \thetaV_t\cdot \x_t = l_f(\x_t; \y_t) \geq  f(\optxv).}
The above inequality states that 
the optimum of \BwR~ is bounded above by what the algorithm gets for \OCO. 
We next show that the optimum value for \OCO~ is equal to what the algorithm of \BwR~ gets, so there is a flip.
This will produce bound on $f(\optxv) - f(\avgT \x)$ in terms of $\regOCO(T)$. 
 
Note that for a fixed $\thetaV$, $g_t$'s differ only in the linear term, so the average of $g_t$'s for all $t$
is equal to  $f^*(\thetaV) - \thetaV \cdot \avgT \x$.
Then, minimizing this over all $\thetaV$ gives $f(\avgT \x)$, by  Lemma \ref{lem:FenchelDuality}. 
%which is again used in the last equality below. 
\EQ{0in}{0in}{ 
%\begin{array}{rcl}
\min_{||\thetaV||_*\le L} \frac{1}{T} \sum_t g_t(\thetaV) 
%& = & \min_{||\thetaV||_*\le L} \frac{1}{T}\sum_t f^*(\thetaV) - \thetaV\cdot \x_t  \\
%& =  & 
= \min_{||\thetaV||_*\le L}f^*(\thetaV)-  \thetaV\cdot\avgT \x = f(\avgT \x).
%\end{array}
}
These two observations together give
\EQ{0in}{0in}{
\begin{array}{rcl}
f(\optxv) - f(\avgT \x) & \leq  & \frac{1}{T}\sum_t g_t(\thetaV_t) -\min_{||\thetaV||_*\le L}  \frac{1}{T}\sum_tg_t(\thetaV) 
=  \frac{1}{T} \regOCO(T).
\end{array}
}
%------------------------------------------------------------
\paragraph{Algorithm for \BwC} 
The algorithm for \BwC~ is obtained by letting $f(\x) = - d(\x, S)$. Note that for this function $L=1$. 
Also, it can be shown that $f^*(\thetaV) = h_S(\thetaV)$, therefore 
$g_t(\thetaV):=h_S(\thetaV) - \thetaV\cdot \x_t.$ And, using the same calculations as in above, we will get 
\EQ{0in}{0in}{d(\avgT{\x}, S) \le \frac{1}{T} \regOCO(T).}
%The gradient of $h_S(\thetaV)$ can be obtained by solving $\max_{\y \in S} \{ \y \cdot \theta\}.$ 
%------------------------------------------------------------
This and (\ref{eq:regCalc2}) imply the following theorem. 
\begin{theorem} 
\label{th:dual}
With probability $1-\delta$, the regret of Algorithm \ref{algo:Dual} is bounded as
\EQ{0in}{0in}{\areg_1(T) = O(L \normOned \sqrt{\frac{\crad m}{T}} + \frac{\regOCO(T)}{T}) \text{ for \BwR, and}}
\EQ{0in}{0in}{\areg_2(T) = O(\normOned \sqrt{\frac{\crad m}{T}}  + \frac{\regOCO(T)}{T}) \text{ for \BwC},}
when used with $f(\x)=-d(\x,S)$. Here $\crad=\cradVal$, and
%Algorithm \ref{algo:Dual} is $\regB$-optimal, with 
%\begin{center} $\regB=\normOned \sqrt{m T \log(\frac{mTd}{\delta})}+ \regOCO(T),$ \end{center} 
$\regOCO(T)$ is the regret for the \OCO~method used. %for convex function $g_t$ defined on a $1$-Lipschitz function.

%With probability $1-\delta$,  the regret of Algorithm \ref{algo:Dual} is bounded as
%$$\textstyle{\areg_1(T) = O(L \normOned \textstyle{\sqrt{\frac{\crad m}{T}} + \frac{\regOCO(T)}{T}}), \ \areg_2(T) = O(\normOned \textstyle{\sqrt{\frac{\crad m}{T}}  + \frac{\regOCO(T)}{T}})},$$
%where $\regOCO(T)$ is the regret for the \OCO~ method used.
%\commentShipra{TODO: make theorem statements consistent for all problems}
\end{theorem} 

In case of Eucledian norm, online gradient descent (OGD) can be used to get $\regOCO(T)=\tilde{O}(GD\sqrt{T})$, where $G$ is an upper bound on Eucledian norm of subgradient of $g_t$, and $D$ is an upper bound on Eucledian norm of $\thetaV$ (refer to \citet{Zinkevich03}, and Corollary 2.7 in \citet{Shalev-Shwartz12}). 
For our purpose, $G \le \sqrt{d}$ and  $D\le L$.
%depends on relation of the norm being used in the constraint $||\theta||_* \le L$ to the Eucledian norm. For Eucledian norm $D\le ||\thetaV||_2 \le L$, and for $L_{\infty}$ norm, the dual norm is $L_1$, so that, $D \le ||\thetaV||_2 \le ||\thetaV||_1 \le L$ ($L=1$ for distance function), and we get $\regOCO(T)\le \tilde{O}(L \sqrt{dT})$. 
For other norms FoRel algorithm with appropriate regularization may provide better guarantees. For example, when $||\cdot||$ is $L_{\infty}$ norm (i.e., $||\cdot||_*$ is $L_1$), 
%we the condition $||\theta||_*\le L$ translates to $L_1$ norm bounded by $L$, 
we can use FoRel algorithm with Entropic regularization (essentially a generalization of the Hedge algorithm \cite{Freund1995}), to obtain an improved bound of $O(L\sqrt{T\log(d)})$ on $\regOCO(T)$ (refer to Corollary 2.14 in \citet{Shalev-Shwartz12}).
%, e.g., $\regOCO(T)=\sqrt{{\crad dT}}$ for Online Gradient Descent method and norm $L_1$ or $L_{\infty}$. % Here, $b=\max\{1, \frac{||{\bf 1}_d||_2}{||{\bf 1}_d||_*}\}$.
%\commentShipra{What do specific algos like Hedge obtain for $L_1$ norm?\\}

%
%\begin{eqnarray*}
%\frac{{\cal R}^c(T)}{T} & = & \frac{1}{T}\sum_t h_t(\theta_t) - \frac{1}{T}\min_{||\theta||\le L} h_t(\theta)\\
%& \ge &  f(p^*W) - f(\frac{1}{T}\sum_t \dis_t\tilde{W}_t)\\
%\end{eqnarray*}
%\textcolor{red}{Do we have a lemma for this?} 
%From this, we can bound the regret for \BwR as 
%\begin{eqnarray*}
%\frac{{\cal R}_1(T)}{T} & = &  f(\mcm\optdis) - f(\frac{1}{T}\sum_t \mcm\dis_t)\\
%& \le & f(\mcm\optdis) - f(\frac{1}{T}\sum_t \mcme_t\dis_t) + \frac{L}{T}||\sum_t (\mcme_t\dis_t -\mcm\dis_t)||\\
%& \le & \frac{{\cal R}^c(T)}{T} + \frac{L}{T}||\sum_t (\mcme_t\dis_t -\mcm\dis_t)||
%\end{eqnarray*}
\paragraph{Implementability}
\OCO~algorithms like online gradient descent require gradient computation.
In this case, we need to compute the gradient of the dual $f^*$ (that is why we call it the dual algorithm) which can be computed as $\arg \max_{\y} \{ \thetaV\cdot \y+ f(\y) \} .$  for a given $\thetaV$.

%\subsection{Bandits with convex knapsacks}

%\textcolor{red}{ We should state the regret bounds that follows from standard algorithms.} 

%%%%%%%%%%%%%%%%%%%%%%%%%%%%%%%%%%%%%%%%%%%%%%%%%%%%%%%%%%%%%%%%%%%%%%%%%%%%%%%

%%%%%%%%%%%%%%%%%%%%%%%%%%%%%%%%%%%%%%%%%%%%%%%%%%%%%%%%%%%%%%%%%%%%%%%%5

%% file: frankwolfe.tex
\subsection{The primal algorithm} 
%\subsection{Concave rewards}
The algorithm presented in Section \ref{sec:dual} required computing the gradient of the Fenchel dual $f^*$ which may be computationally expensive in some cases. %(This is why it was called the dual algorithm.) 
%In Section \ref{} we presented an algorithm that only required computing the gradient of $f$, but at multiple points in each step. 
Here we present a primal algorithm (for \BwR) that requires computing the  gradient of $f$ in each step, 
based on the Frank-Wolfe algorithm \cite{frank-wolfe}.  
A caveat is that  this requires a stronger 
assumption on $f$, that $f$  is smooth in the following sense. 
\begin{assumption}
\label{assum:Csmooth}
We call concave function$f(\cdot)$ to be $\sC$-smooth if 
%defined on a convex set $Q$, define $C_{f,Q}$ as the minimal value of
\begin{equation}
\label{eq:cfq} 
 f(\z + \alpha(\y-\z)) \ge f(\z) + \alpha \nabla f(\z)\cdot (\y-\z) - \frac{\sC}{2} \alpha^2,
\end{equation} 
 for all $\y, \z \in [0,1]^d$ and $\alpha \in [0,1]$.
If $f$ is such that the {\em gradient} of $f$ is Lipshitz continuous (with respect to any $L_q$ norm) with a constant $G$, then $\sC \le G d$. 
%this holds irrespective of the norm with respect to which gradient of f is Lipschitz continuous, when the domain is [0,1]^d because ||1_d||\cdot ||1_d||_*=d.
%The derivation is by f(y)= f(x)+\int_{t=0^1} \grad f(tx+(1-t)y) (x-y) dt \ge f(x) + \grad fx(y-x) + (\grad f(tx+(1-t)y)- \grad f(x)) (x-y) dt
% the last term is \ge -G||x-y|| ||x-y||_* \int_t (tdt)
% And ||x-y|| ||x-y||_* \le ||1_d|| ||1_d||_* \le d
% where $D$ is the diameter of $Q$ and $G$ is the Lipschitz constant for the \emph{gradient} of $f$. 
%For the convex set $\{\mcme\dis: \dis\in \Delta_m, \mcme\in \cHypercube_t\}$, the diameter $D \le \sqrt{m}$. 
\end{assumption}
Note that the distance function ($f(\z)=-d(\z,S)$) does not satisfy this assumption.
\commentShipra{This assumption may be stronger than Assumption \ref{assum:Lcont}. For instance, for distance function ($f(\x) = - d(\x,S)$) the Assumption \ref{assum:Lcont} is satisfied with $L=1$, but not Assumption \ref{assum:Csmooth}. 
In Section \ref{sec:nonsmooth-fw}, we show how to use the 
technique of \cite{nesterov} to convert a non-smooth $f$ that only satisfies Assumption \ref{assum:Lcont} into one that 
satisfies Assumption \ref{assum:Csmooth}.  %Implementability? 
Interestingly, for the smooth approximation of distance function, this algorithm will have essentially the same structure as the 
(primal) algorithm for the Blackwell approachability problem, thus drawing a connection between two well known algorithms. 
}
%\commentShipra{(not sure if this is true) Additionally, our analysis of Frank-Wolfe algorithm will apply to distance with respect to any norm, whereas the techniques used for the analysis of primal algorithm for the Blackwell approachability problem apply only to the Eucledian norm (refer to \cite{}).}
%The dual based algorithm requires computing the gradient of $f^*$. 
%In this section, we present another algorithm use Frank-Wolfe technique to obtain an algorithm that only 
%requires the gradient of $f$, which may be more efficient depending on how 
%$f$ is represented. 

Like the Fenchel dual based algorithm, in each step, this algorithm too picks a $\thetaV_t$ and 
sets $\dis_t$ according to (\ref{eq:ptaslinearopt}). The difference is that 
$\thetaV_{t} $ is now simply $-\nabla f(\avgtm \x)$! 
%We denote the average of these vectors  up to time $t$ as 
%$\avgxv_t := \tfrac 1 t\sum_{\tau=1}^t \mcme_\tau \dis_\tau$. Then $\theta_t = -\nabla f(\avgxv_{t-1})$. 

\begin{algorithm}[H] 
\label{algo:Primal}
\caption{Frank-Wolfe based primal algorithm for \BwR}
  \begin{algorithmic}
\FORALL{$t=1, 2,\ldots, T$} 
		%\STATE Set $\tilde{W}_{ij}(t)=\UCB(W_{ij})$ for $\theta_j<0$, and $\tilde{W}_{ij}(t)=\LCB(W_{ij})$ for $\theta_j\ge 0$.
		\STATE 	\EQ{0in}{0in}{(\dis_t, \mcme_t)=\arg \max_{\dis\in \Delta_m} \max_{\mcme \in \cHypercube_t} (\mcme\dis) \cdot \nabla f(\avgtm \x),}
		%where $\avgtm \x = \frac{1}{t-1}\sum_{\tau=1}^{t-1} \x_{\tau}$, and $\x_{\tau}=\mcme_{\tau}\dis_{\tau}$.
		where $\x_t=\mcme_t\dis_t$. Play arm $i$ with probability $\disS_{t,i}$.
\ENDFOR
%  	}
  	\end{algorithmic}
\end{algorithm}

%\comment{
%\begin{lemma}
%For concave function $f(\cdot)$ on convex set $Q$, define $C_{f,Q}$ as the minimal value of $C$ statisfying
%$$ f(x + \alpha(y-x)) \ge f(x) + \alpha \nabla f(x)\cdot (y-x) - \frac{C}{2} \alpha^2, \forall \alpha \in [0,1].$$
%If $f$ is smooth, then $C_{f, Q} \le G D^2$ where $D$ is the diameter of $Q$ and $G$ is the Lipschitz constant for the gradient of $f$. 
%\end{lemma}
%\begin{proof}
%By fundamental theorem of calculus 
%\begin{eqnarray*}
% f(x + \alpha(y-x)) & = & f(x) + \alpha \nabla f(x)\cdot (y-x) + \int_0^1 (\nabla f(x + z\alpha(y-x)) - \nabla f(x)) \cdot \alpha(y-x)\ dz\\
%& \le & f(x) + \alpha \nabla f(x)\cdot (y-x) + \int_0^1 ||\nabla f(x + z\alpha(y-x)) - \nabla f(x)|| \cdot \alpha ||y-x||\ dz\\
%& \le & f(x) + \alpha \nabla f(x)\cdot (y-x) + \int_0^1 \alpha \ dz\\
%\end{eqnarray*}
%\end{proof}
%}

\begin{theorem} 
\label{th:primal}
With probability $1-\delta$, the regret of Algorithm \ref{algo:Primal} for \BwR~problem with $\sC$-smooth function $f$ (Assumption \ref{assum:Csmooth}), is bounded as
\EQ{0in}{0in}{\areg_1(T) = O(L \normOned \textstyle{\sqrt{\frac{\crad m}{T}} + \frac{\sC\log(T)}{T}}).}
%where $\crad=\cradVal$. 
%with probability $1-\delta$.
\end{theorem}
\begin{proof}
Using \eqref{eq:regCalc2}, proving this regret bound essentially means bounding $f(\optxv) - f(\avgT  \x)$. This quantity can be bounded by $\frac{\sC\log(2T)}{2T}$ using techniques similar to those used in the analysis of Frank-Wolfe algorithm for convex optimization \cite{frank-wolfe}. The complete proof is provided in Appendix \ref{app:frankWolfe}.
\end{proof}

%%%%%%%%%%%%%%%%%%%%%%%%%%%%%%%%%%%%%%%%%%%%%%%%%%%%%%%%%%%%%%%%%%%%5

\subsection{Smooth approximation of Non-smooth $f$}
\label{sec:nonsmooth-fw}
Assumption \ref{assum:Csmooth} may be stronger than Assumption \ref{assum:Lcont}. For instance, for distance function ($f(\z) = - d(\z,S)$) Assumption \ref{assum:Lcont} is satisfied with $L=1$, but not Assumption \ref{assum:Csmooth}. In this section, we show how to use the technique of \cite{nesterov} to convert a non-smooth $f$ that only satisfies Assumption \ref{assum:Lcont} into one that satisfies Assumption \ref{assum:Csmooth}.   For simpicity, we assume $||\cdot||$ to be Eucledian norm in this section. 
Interestingly, for the smooth approximation of distance function, this algorithm will have essentially the same structure as the 
(primal) algorithm for the Blackwell approachability problem, thus drawing a connection between two well known algorithms. 
%In this section we show how to use smooth approximation techniques %the techniques of \cite{flaxman2005} and \cite{nesterov} to 
%convert functions that satisfy Assumption \ref{assum:Lcont} to ones 
%that satisfy Assumption \ref{assum:Csmooth}, so that Algorithm \ref{algo:Primal} can also be used for functions that satisfy only 
%Assumption \ref{assum:Lcont}. %This will lead to an interesting connection to the algorithm for Blackwell approachability problem. 
%For any $f$, we can define its smooth approximation $\hat{f}_{\sP}$ with parameter $\sP$ 
%in the following two different ways, given by \cite{flaxman2005} and 
%using technique of \cite{nesterov}. 
%For simplicity, we consider only Eucledian norms in this section.
%Results by \cite{HazanKale2012} and by \cite{nesterov}, respectively, formalize the notion that $\hat{f}_\lambda$ defined as in \eqref{eq:fsmooth1} and \eqref{eq:fsmooth2}  are smooth approximations of $f$. 
%We present a proof in Appendix \ref{app:smooth} for completeness. 
\begin{theorem}{\cite{nesterov}} 
\label{th:Nesterov}
Define 
\begin{equation} 
\label{eq:fsmooth2} 
 \textstyle{\hat{f}_{\sP}(\z) := \min_{||\thetaV|| \leq L} \{ f^*(\thetaV) + \frac{\sP}{2L} \thetaV \cdot \thetaV  - \thetaV\cdot \z\} .}
\end{equation} 
Then, $\hat{f}_\sP$ is concave, differentiable,  and $\frac{d L}{\sP}$-smooth. %$\nabla f_\mu $ is $L/\sP$-Lipshitz continuous. 
Further, $\hat{f}_\sP - \frac{\sP}{2} L\leq f \leq \hat f_\sP.$
\end{theorem} 

 %Since $\nabla f_\mu $ is $1/\mu$-Lipshitz continuous, the function $f_\mu$ satisfies Assumption \ref{assum:Csmooth} with  $\sC = m/\mu$. 
%Now, if we run Algorithm \ref{algo:Primal} on $\hat{f}_\sP$ for \BwR~ (or, on $\hat{f}_{\sP}(\z)=-\hat{d}_{\eta}(\z,S)$ for \BwC), we get the following regret bound with $\sP=\sqrt{\frac {d \log(2T)} {T}}$.
Now, if we run Algorithm \ref{algo:Primal} on $\hat{f}_\sP$, with $\sP=\sqrt{\frac{d}{T}\log(2T)}$, we get that 
%$\hat{f}_\sP(\optxv) - \hat{f}_\sP(\avgT \x) \leq \sC\log(2T)/2T$. 
%From Theorem \ref{th:Nesterov}, this implies that 
$f(\optxv) - f(\avgT \x) \leq \frac{\sC\log(2T)}{2T} + \frac{\sP}{2} L \le \frac{L}{2} \sqrt{\frac {d \log(2T)} {T}}$.
 %by setting  $\sP =\sqrt{\frac {d \log(2T)} {T}}.$
The algorithm and regret bound for \BwC~can be obtained similarly by using this smooth approximation for distance function, i.e., for $f(\z)=-d(\z,S)$. %, $\hat{f}_{\sP}(\z)=-\hat{d}_{\eta}(\z,S)$. 
We thus obtain the following theorem.
%------------------------------------------------------ 
\begin{theorem} 
\label{th:primal-nonsmooth}
With probability $1-\delta$, the regret of Algorithm \ref{algo:Primal} when used with smooth approximation $\hat{f}_{\sP}(\z)$ of function $f(\z)$ (or, $-\hat{d}_{\sP}(\z, S)$ of function $-d(\z,S)$), is bounded as \vspace{-0.1in}
\EQ{0in}{0in}{\areg_1(T) = O(L \normOned \sqrt{\frac{\crad m}{T}} + L \sqrt{\frac {d \log(T)} {T}}) \text{ for \BwR, and}}
\EQ{0in}{0in}{\areg_2(T) = O(\normOned \sqrt{\frac{\crad m}{T}}  + \sqrt{\frac {d \log(T)} {T}}) \text{ for \BwC}.}
%were $\crad=\cradVal$.
%Algorithm \ref{algo:Primal}, when used with smooth approximation of function $f(\z)$ (or $-d(\z,S)$) is $\regB$-optimal, with $\regB=\normOned \sqrt{m T \log(\frac{mTd}{\delta})} + \sqrt{d \log(\frac{T}{\delta})}$.
% when used with smooth approximation of function $f(\z)$ (or distance function $d(\z,S)$).
%With probability $1-\delta$,  the regret of Algorithm \ref{algo:Primal}, when used with smooth approximation of function $f(\z)$ (or distance function $d(\z,S)$), is bounded as
%$$\areg_1(T) = O(L \normOned \textstyle{\sqrt{\frac{\crad m}{T}} + L \sqrt{\frac{d \log(T)}{T}}}),~\areg_2(T) = O(\normOned \textstyle{\sqrt{\frac{\crad m}{T}} + \sqrt{\frac{d \log(T)}{T}}}),$$ 
%with probability $1-\delta$, where $\crad = \cradVal$, $\sP=\sqrt{\frac {d \log(T)} {T}}$.
\end{theorem} 
%--------------------------------------------------
For the distance function, this smooth approximation has some nice characteristics.
\begin{lemma}
\label{lem:distance-smooth}
For the distance function $d(\z,S)$, \eqref{eq:fsmooth2} provides smooth approximation $ \hat{d}_\sP(\z,S) = \max_{||\thetaV|| \le 1}\thetaV\cdot \z - h_S(\thetaV) - \frac{\sP}{2} \thetaV\cdot \thetaV,$ 
%where $h_S(\thetaV)=\max_{\s \in S} \thetaV \cdot \s$, as defined earlier. 
and, the gradient of this function is given by 
\[ \nabla \hat{d}_\sP (\z) =  \threepartdef 
	{\frac {\z - \proj_S(\z)} {||\z - \proj_S(\z)||} }  	{ ||\z  -\pi_S(\z)|| \geq \sP } 
	{\frac {\z - \proj_S(\z)}  \sP}			{ 0 < ||\z  -\pi_S(\z)|| < \sP}
	{\mathbf 0} 			{\z \in  S},
\] 
where $\proj_S(\z)$ denotes the projection of $\z$ on $S$.
\end{lemma} 
The proof of the above lemma along with a proof of Theorem \ref{th:Nesterov} is in Appendix \ref{app:frankWolfe}. 
Note that for Algorithm \ref{algo:Primal} only the direction of the gradient of $f$ matters, and in this case the direction of gradient of $f=-\hat{d}_{\sP}$ at $\z$ is 
$-(\z - \proj_S(\z))$ for all $\z \notin S$. For $\z \in S$, the gradient is $\mathbf 0$, which means it does not really matter what $\dis$ is picked. Therefore, 
%for the smooth approximation of distance function, the 
Algorithm \ref{algo:Primal} reduces to the following.

\begin{algorithm}[H] 
\label{algo:Primal-BwC}
\caption{Frank-Wolfe based primal algorithm for \BwC}
  \begin{algorithmic}
\FORALL{$t=1, 2,\ldots, T$} 
		%\STATE Set $\tilde{W}_{ij}(t)=\UCB(W_{ij})$ for $\theta_j<0$, and $\tilde{W}_{ij}(t)=\LCB(W_{ij})$ for $\theta_j\ge 0$.
		\STATE 	If $\avgtm \x \in S$, set $\dis_t$ arbitrarily.
		\STATE 	If $\avgtm \x \notin S$, find projection $\proj_S(\avgtm \x)$ of this point on $S$. And compute
		\begin{center}$(\dis_t, \mcme_t)=\arg \min_{\dis\in \Delta_m} \min_{\mcme \in \cHypercube_t} (\mcme\dis) \cdot (\avgtm \x - \proj_S(\avgtm \x)),$\end{center}
		%where $\x_t=\mcme_{t}\dis_{t}$.
		\STATE Play arm $i$ with probability $\disS_{t,i}$.
\ENDFOR
%  	}
  	\end{algorithmic}
\end{algorithm}
Algorithm \ref{algo:Primal-BwC} has the same structure as the Blackwell's algorithm for the approachability problem \cite{blackwell1956}, which asks to play anything at time $t$ if $\avgtm \x$ is in $S$. Otherwise, find a point $\x_t$ such that $\x_t-\avgtm \x$ makes a negative angle with $(\avgtm \x-\proj_S(\avgtm \x))$. %The latter can be achieved by finding $\x_t$ that minimizes $\x_t \cdot (\avgtm \x-\proj_S(\avgtm \x))$. 
We have $\x_t=\mcme_t\dis_t$. However, the proof of convergence of Blackwell's algorithm as given in \cite{blackwell1956} seems to be different from the proof derived here, via the smooth approximation and 
Frank-Wolfe type analysis. This gives an interesting connection between well known algorithms, Blackwell's algorithm for the approachability problem and Frank-Wolfe algorithm for convex optimization, via Nesterov's method of smooth approximations!!

\paragraph{Implementability} The algorithm with smooth approximation needs to compute the gradient of $\hat{f}_\sP$ in each step and 
in general there is no easy method to compute this, except in some special cases like the distance function discussed above. Alternatively, one could use the smooth approximation $\hat{f}_{\sP}(\z)=\Ex_{\bs{u} \in \mathbb{B}}[f(\z+\delta \bs{u})]$ given by \cite{Flaxman2005}, which has slightly worse smoothness coefficient but has easy-to-compute gradient by sampling.

%\textcolor{red}{Should we have a theorem stating the regret bounds for this case, this algo for \BwC?} 

%% file: efficient-BwCR.tex
\section{Computationally efficient algorithms for \BwCR}
\label{sec:eff-BwCR}

Any combination of the primal and dual approaches mentioned in the previous sections can be used to get an efficient algorithm for the \BwCR~ problem. Using the observations in Equation \eqref{eq:ptaslinearopt} and \eqref{eq:MinEst}, we obtain an algorithm with the following structure.

\begin{algorithm}[H] 
\caption{Efficient algorithm for \BwCR}
\label{algo:eff-BwCR}
  \begin{algorithmic}
	\STATE Inititalize $\thetaV_{1}$. 
\FORALL{$t=1, 2,\ldots, T$ \vspace{-0.2in}} 
		%\STATE Set $\tilde{W}_{ij}(t)=\UCB(W_{ij})$ for $\theta_j<0$, and $\tilde{W}_{ij}(t)=\LCB(W_{ij})$ for $\theta_j\ge 0$.
		\STATE 
		\begin{equation}
		\label{eq:algoChoice-eff}
		\dis_t =\begin{array}{rcl}
		\arg \min_{\dis\in \Delta_m} & (\thetaV_t \cdot \vertex{\thetaV_t}) \dis  &\\
		\text{s.t.} & (\phiV_t \cdot \vertex{\phiV_t})\dis \le h_S(\phiV_t).&
		\end{array}
		\end{equation}
		\STATE Play arm $i$ with probability $p_{t,i}$. Compute $\thetaV_{t+1}$, $\phiV_{t+1}$.
\ENDFOR
%  	}
  	\end{algorithmic}
\end{algorithm}
Here, $\vertex{\cdot}$ is a vertex of $\HC_t$ as defined in \eqref{eq:MinEst}. 
Now, either primal or dual approach can be used to update $\thetaV$, irrespective of what approach is being used for updating $\phiV$, and vice-versa. The choice between primal and dual approach will depend on properties of $f$ and $S$, e.g., whether it is easy to compute the gradient of $f$ or its dual $f^*$. 
%If computing the gradient of Fenchel dual $f^*$ is easy, one may use the dual approach as in Algorithm \ref{algo:Dual} and update $\thetaV$ using an \OCO~ based update. If $f$ is smooth and computing the gradient of $f$ is easy compared to computing the gradient of $f^*$, one may use the primal approach and set $\thetaV_t = -\nabla f(\avgtm \x)$. For non-smooth $f$, one can use $\thetaV_t = -\nabla \hat{f}_{\sP}(\avgtm \x)$. 
%Similarly, irrespective of the approach being used for $\thetaV$, either primal or dual approach can be used for updating $\phiV$. The dual approach amounts to updating $\phiV$ using \OCO. And, the primal Frank-Wolfe approach (refer to Algorithm \ref{algo:Primal-BwC}) amounts to setting $\phiV_t=\proj_S(\avgtm \x)-\avgtm \x$. 
It is easy to derive regret bounds for this efficient algorithm for \BwCR~using results in the previous section.
\begin{theorem}
For Algorithm \ref{algo:eff-BwCR}, $\areg_1(T)$ is given by Theorem \ref{th:dual}, Theorem \ref{th:primal}, or Theorem \ref{th:primal-nonsmooth}, respectively, dependening on whether the dual, primal, or primal approach with smooth approximation is used for updating $\thetaV$. And, $\areg_2(T)$ is given by Theorem \ref{th:dual} or Theorem \ref{th:primal-nonsmooth}, respectively, dependening on whether the dual or primal approach is used for updating $\phiV$.
\end{theorem}

\comment{
\begin{proof}
The proof simply follows from observing that the following version of Equation \eqref{eq:cor2} still holds with high probability ($1-(mTd)e^{-\Omega(\crad)}$).
\begin{equation}
\label{eq:cor2-BwCR}
\begin{array}{c}
l_f(\x_t;\y_t) \ge l_f(\optxv;\y_t) \geq f(\optxv) ,  \\
\z_t\in H_S(\phiV_t),
\end{array}
\end{equation}
where $\x_t = \vertex{\thetaV_t} \dis_t$, $\y_t$ is such that $\thetaV_t=-\nabla f(\y_t)$, and $\z_t=\vertex{\phiV_t} \dis_t$.
This holds simply because $\dis^*$ is still a feasible solution at time step $t$ with high probability. More precisely, $\mcm \dis^* \in S \subseteq H_S(\phiV_t)$, so that,
$$(\vertex{\phiV_t}\dis^*) \cdot \phiV_t \le (\mcm \dis^*) \cdot \phiV_t \le h_S(\phiV_t).$$
Then, the inequalities hold by construction.
\end{proof}
}
One can also substitute the constraint or objective in \eqref{eq:algoChoice-eff} by the corresponding expression in Equation \eqref{eq:algoChoice} of the UCB algorithm, if efficiency is not as much of a concern as regret for either constraint or objective.

\paragraph{Implementability} Every step $t$ of Algorithm \ref{algo:eff-BwCR} requires solving a linear optmization problem over simplex with one additional linear constraint. This is a major improvement in efficiency over Algorithm \ref{algo:UCB-BwCR}, which required solving a difficult convex optimization problem over domain $\{\dis\mcme: \mcme \in \HC_t, \dis\in \Delta_m\}$. 

Also, Algorithm \ref{algo:eff-BwCR}  is particularly simple to implement when the given application allows {\em not} playing {\em any} arm at a given time step, i.e. relaxing the constraint $\sum_i p_i=1$ to $\sum_i p_i \le 1$. This is true in many applications, for example, an advertiser is allowed to not participate in a given auction. In particular, in \BwK, the algorithm can abort at any time step, effectively chosing not to play any arm in the remaining time steps. In such an application, if further $h_S(\phiV_t) \ge 0, \phiV_t^T\vertex{\phiV_t}_i > 0, \forall i$, \eqref{eq:algoChoice-eff} is a special case of fractional knapsack problem, and the greedy optimal solution in this case reduces to simply choosing the arm $i$ that minimizes $\frac{\thetaV_t^T\vertex{\thetaV_t}_i}{\phiV_t^T\vertex{\phiV_t}_i}$, and playing it with probability $p$, where $p\in[0,1]$ is the highest value satisfying $\phiV_t^T\vertex{\phiV_t}_i p \le h_S(\phiV_t)$. Even if $\phiV_t^T\vertex{\phiV_t}_i \le 0$ for some arms $i$, some simple tweaks to this greedy choice work. \commentShipra{, which would require finding the one with best objective value among these arms separately and comparing it to the greedy choice.}

In the special case of \BwK, it is not difficult to compute that $-\thetaV_t^T\vertex{\thetaV_t}_i = \UCB_{t,i}(\mrv)$, $\phiV_t^T\vertex{\phiV_t}_i  = \phiV_t^T\LCB_{t,i}(\Bmcm)$, and $h_S(\phiV_t)=\frac{B}{T}$ for all $t$, so that the above greedy rule simply becomes that of selecting arm 
\EQ{0in}{0in}{\tarm=\arg \max_i \frac{\UCB_{t,i}(\mrv)}{\phiV_t^T\LCB_{t,i}(\Bmcm)},}
and playing it with largest probability $p$ such that $\phiV_t^T\LCB_{t,i}(\Bmcm) p \le \frac{B}{T}$.
This is remarkably similar to the PD-BwK algorithm of \citet{BwK}, except that their algorithm plays this greedy choice with probability $1$ and aborts when any constraint is violated. 
\commentShipra{Our algorithm may choose to skip playing any arm at some time steps, but will maintain that with high probability constraints are satisfied throughout time $T$.}

%The latter is a complicated domain, computing separating hyperplane for its intersection with $S$ or optimizing $f$ over this domain was a computationally intensive task, though solvable in polynomial time by elliposid method. 
%\commentShipra{TODO: special algos for linear opt over simplex. Algo in BwK}

%% file: appendix.tex
%=========================================================================
\section{Preliminaries}
\label{app:prelims}

\begin{proof}[of Lemma \ref{lem:benchmark-BwCR}]
For a random instance of the problem, let $\tilde{\disS}_i$ denote the empirical probability of playing arm $i$ in the {\em optimal instance specific solution in hindsight}, and $\cv_t$ denote the observation vector at time $t$. Then, it must be true that $\frac{1}{T}\sum_t \cv_t \in S$.
Let $\dis^* = \Ex[\tilde{\dis}]$.  Then, 
$$\Ex[\frac{1}{T} \sumT_{t} \cv_t] = \frac{1}{T} \Ex[\sumT_{t} \Ex[\cv_t | \tarm]] =\Ex[\mcm \tilde{\dis}_t] = \mcm \dis^*.$$
So that,  due to convexity of $S$, $\frac{1}{T}\sum_t \cv_t \in S$ implies that $\mcm \dis^* \in S$. And, by concavity of $f$,
$$\OPT_f \le \Ex[f(\frac{1}{T} \sum_{t} \cv_t)] \le f(\Ex[\frac{1}{T} \sum_{t} \cv_t]) =f(\mcm\dis^*). $$
\end{proof}
%------------------------------------------------------------------------------
%\section{Fenchel Duality}
\label{app:Fenchel}

\begin{proof}[of Lemma \ref{lem:FenchelDuality}]
\begin{eqnarray*}
\min_{||\thetaV||_*\le L} f^*(\thetaV)-\thetaV\cdot \z & = & \min_{||\thetaV||_* \le L} \max_{\y} \{ \y\cdot \thetaV + f(\y) -\thetaV\cdot \z\}\\
& = & \max_{\y} \min_{||\thetaV||_*\le L} \{ \y\cdot\theta+ f(\y) -\thetaV\cdot \z\}.
\end{eqnarray*}
The last equality uses minmax theorem. Now, by our assumption, for any $\z$, there exists a vector $||\g_z||_* \le L$ which is a supergradient of $f$ at $\z$, i.e.,
$$f(\y) - f(\z) \le \g_z \cdot(\y-\z), \forall \y.$$
Therefore, for all $\y$,
$$\min_{||\theta||\le L}\{ \y\cdot \theta + f(\y) -\theta\cdot \z \} \le (-\g_z)\cdot \y + f(\y) - (-\g)\cdot \z \le f(\z),$$
with equality achieved for $\y=\z$. 
\end{proof}

%%%%%%%%%%%%%%%%%%%%%%%%%%%%%%%%%%%%%%%%%%%%%%%%%%%%%%%%%%%%%%%%%%%%%%%%%%%%%%%%%5

\section{UCB family of algorithms}
\label{app:UCB}
%------------------------------------------------------------------------
 We will use the following concentration theorem.

\begin{lemma}{\cite{Kleinberg2008, babaioff2012, BwK}}
\label{lem:concentration}
Consider some distribution with values in $[0,1]$, and expectation $\nu$. Let $\hat{\nu}$ be the average of $N$ independent samples from this distribution. %($\hat{\nu}=0$, if $N=0$). 
Then, with probability at least $1-e^{-\Omega(\crad)}$, for all $\crad>0$, %$1-2e^{-(\crad/72)}$,
\begin{equation}
|\hat{\nu}-\nu| \le \rad(\hat{\nu}, N) \le 3\rad(\nu, N),
\end{equation}
where $\rad(\nu, N)=\sqrt{\frac{\crad \nu}{N}} + \frac{\crad}{N}.$
More generally this result holds if $X_1, \ldots, X_N \in [0,1]$ are random variables, $N\hat{\nu}=\sum_{t=1}^N X_t$, and $N\nu=\sum_{t=1}^N \Ex[X_t | X_1,\ldots, X_{t-1}]$.
\end{lemma}

\begin{lemma} \cite{BwK}
\label{lem:sumrad}
%$$\rad(\nu+ n\  \rad(\nu, N), N) \le (\sqrt{2b}+1)\rad(\nu, N).$$
%$$\sum_{n=1}^N \rad(a, n)  %= \sum_{n=1}^N \sqrt{\frac{\crad a}{n}} + \frac{\crad}{n} = 2\sqrt{\crad a N} + \crad \log(N) +\crad 
%\le 2\log(N) ]N \cdot \rad(a,N).$$

For any two vectors $\bs{a}, \bs{n} \in {\mathbb R}_{+}^m$, 
$$\sum_{j=1}^m \rad(a_j, n_j) n_j \le \sqrt{\crad m (\bs{a}\cdot \bs{n})} + \crad m	.$$
\end{lemma}
%%%%%%%%%%%%%%%%%%%%%%%%%%%%%%%%%%%%%%%%%%%%%%%%%%%%%%%%%%%%%%55
\subsection{\BwCR}
\label{app:UCB-BwCR}

\begin{lemma}
\label{lem:empConc}
Define empirical average $\emp{\mcmS}_{t,ji}$ for each arm $i$ and component $j$ at time $t$ as
\begin{equation}
\label{eq:empV}
 \emp{\mcmS}_{t,ji} = \frac{\sum_{s<t:\sarm=i} \cvS_{t,j}}{k_{t,i}+1},
\end{equation}
where $k_{t,i}$ is the number of plays of arm $i$ before time $t$.
Then, $\emp{\mcmS}_{t,ji}$ is close to the actual mean $\mcmS_{ji}$: for every $i,j,t$, with probability $1-e^{-\Omega(\crad)}$,
$$|\emp{\mcmS}_{t,ji}-\mcmS_{t,ji}|  \le 2\rad(\emp{\mcmS}_{t,ji}, k_{t,i}+1).$$
%$$|\emp{\mcm}_{t,ji}-\mcm_{t,ji}|  \le 4\rad(\mcm_{ji}, k_{t,i}+1).$$
\end{lemma}
\begin{proof}
This proof follows from application of Lemma \ref{lem:concentration}.
We apply Lemma \ref{lem:concentration} to $\cvS_{1,j}, \ldots, \cvS_{T,j}$, for each $j$, using $\Ex[\cvS_{t,j} | \tarm] = \mcmS_{t,j\tarm}$, to get that with probability at least $1-e^{-\Omega(\crad)}$, %$1-2e^{-(\crad/72)} \ge 1-\frac{1}{m(d+1)T^3}$
%\begin{center}
%$|\sum_{\tau:\tarm=i, \tau\le t-1} \cv_{t,j} - k_{t,i} \mcm_{ji}|\le \sqrt{\crad \emp{\mcm}_{t,ji} (k_{t,i}+1)} + \crad.$
%\end{center}
%Therefore,
%The empirical average $\emp{\mcm}_{t,ji}$ for each arm $i$ and component $j$ at time $t$ is close to the actual mean $\mcm_{ji}$.
\begin{eqnarray}
%\label{eq:empConc}
 |\emp{\mcmS}_{t,ji}-\mcmS_{t,ji}| %& = & \frac{1}{k_{t,i}+1}\cdot \left|\sum_{\tau:\tarm=i, \tau\le t-1} \cv_{t,j} - k_{t,i}\mcm_{ji} -\mcm_{ji}\right| \\
& \le & \frac{k_{t,i}}{k_{t,i}+1} \cdot \rad(\emp{\mcmS}_{t,ji}, k_{t,i}) +  \frac{\mcmS_{ji}}{k_{t,i}+1}\nonumber\\
& \le &  \rad(\emp{\mcmS}_{t,ji}, k_{t,i}+1) + \frac{\mcmS_{ji}}{k_{t,i}+1} \nonumber\\
& \le & 2\rad(\emp{\mcmS}_{t,ji}, k_{t,i}+1).\nonumber
\end{eqnarray}
\comment{
\begin{eqnarray}
\label{eq:empConcA}
 |\emp{\mcm}_{t,ji}-\mcm_{t,ji}| & \le & \frac{k_{t,i}}{k_{t,i}+1} \cdot 3\rad({\mcm}_{ji}, k_{t,i}) +  \frac{\mcm_{ji}}{k_{t,i}+1}\nonumber\\
& \le &  3\rad(\mcm_{ji}, k_{t,i}+1) + \frac{\mcm_{ji}}{k_{t,i}+1} \nonumber\\
& \le & 4\rad({\mcm}_{ji}, k_{t,i}+1).
\end{eqnarray}
}
\end{proof}

%--------------------------------------------------------------------------------------
\begin{proof}[of Theorem \ref{th:UCB-BwCR}]
We use the following estimates 
\begin{equation}
\label{eq:app:UCBLCB}
\begin{array}{rcl}
\UCB_{t,ji}(\mcm) = \min\{1,\emp{\mcmS}_{t,ji} + 2 \rad(\emp{\mcmS}_{t,ji}, k_{t,i}+1)\},\\
\LCB_{t,ji}(\mcm) = \max\{0,\emp{\mcmS}_{t,ji} - 2 \rad(\emp{\mcmS}_{t,ji}, k_{t,i}+1)\}, 
\end{array}
\end{equation}
for $i=1,\ldots, m, j=1,\ldots, d, t=1,\ldots, T$. Here $\rad(\nu, N) = \sqrt{ \frac{\crad \nu}{N}} + \frac{\crad}{N}$, $k_{t,i}$ is the number of plays of arm $i$ before time $t$, and $\emp{\mcmS}_{t,ji}$ is the empirical average as defined in Equation \eqref{eq:empV}. These estimates are similar to those used in literature on UCB algorithm for classic MAB and to those used in \cite{BwK}.

Then, using concentration Lemma \ref{lem:concentration}, we will prove that the properties in Equation \eqref{eq:prop1} and \eqref{eq:prop2} hold with probability $1-(mTd)e^{-\Omega(\crad)}$, and with $\diff(T)=O(\normOned\sqrt{\crad mT})$ where $\normOned$ denotes the norm of $d$ dimensional vector of all $1$s. Theorem \ref{th:UCB-BwCR} will then follow from the calculations in Equation \eqref{eq:regCalc}.

Property (1) stated as Equation \eqref{eq:prop1} is obtained as a corollary of Lemma \ref{lem:empConc} by taking a union bound for all $i,j,t$. 
With probability $1-(mTd)e^{-\Omega(\crad)}$,
					$$\UCB_{t,ji}(\mcm) \ge \mcmS_{ji} \ge \LCB_{t,ji}(\mcm), \forall i,j,t.$$

Next, we prove Property (2) stated in Equation \eqref{eq:prop2}. 
%\label{lem:UCBdiff}
Given that arm $i$ was played with probability $\disS_{t,i}$ at time $t$, for any $\{\mcme_t\}_{t=1}^T$ such that $\mcme_{t} \in H_t$ for all $t$, we will show that with probability $1-(mTd)e^{-\Omega(\crad)}$, 
$$||\sum_{t=1}^T (\mcme_t \dis_t - \cv_t) || = O(\normOned \sqrt{\crad mT}).$$
We use the observation that $\Ex[\cv_t | \tarm] = \mcm_{\tarm}$. Then, using concentration Lemma \ref{lem:concentration}, with probability $1-de^{-\Omega(\crad)}$
\begin{equation}
\label{eq:tmp1}
|\sum_{t=1}^T (\mcmS_{j\tarm}-\cvS_{t,j})| \le 3 \rad(\frac{1}{T}\sum_{t=1}^T \mcmS_{j\tarm}, T) =O(\sqrt{\crad T}),
\end{equation}
for all $j=1,\ldots, d$. %where by $|\bs{\omega}|$ for a vector $\bs{\omega}$, we denote the component wise absolute value.
Therefore, it remains to bound $\sumT_t (\mcme_t \dis_t- \mcm_{\tarm})$. Again, since $\Ex[\mcme_{t,\tarm}| \dis_t, \mcme_t] = \mcme_t \dis_t$, we can obtain, using Lemma \ref{lem:concentration},
\begin{equation}
\label{eq:tmp2}
|\sum_{t=1}^T (\mcmeS_{t, j\tarm}-\mcme_{t,j} \dis_t) |\le 3\rad(\frac{1}{T}\sum_{t=1}^T \mcme_{t,j} \dis_t, T) = O(\sqrt{\crad T}),
\end{equation}
for all $j$ with probability $1-de^{-\Omega(\crad)}$. 
Now, it remains to bound $|\sum_{t=1}^T (\mcmeS_{t,j\tarm}-\mcmS_{j\tarm})|$. Using Lemma \ref{lem:empConc}, with probability $1-(mTd)e^{-\Omega(\crad)}$, for all $i,j,t,$
$$|\emp{\mcmS}_{t,ji}-\mcmS_{t,ji}|  \le 2\rad(\emp{\mcmS}_{t,ji}, k_{t,i}+1).$$
Applying this, along with the observation that for any $\mcme \in \HC_t$, $\LCB_{t,ij}(\mcm) \le \mcmeS_{t,ji} \le \UCB_{t,ji}(\mcm)$, we get
\begin{eqnarray}
\label{eq:tmp3}
|\sumT_{t=1}^T (\mcmeS_{t,j\tarm}-\mcmS_{j\tarm})| & \le &  \left(\sum_t 4 \rad(\emp\mcmS_{t,j\tarm}, k_{t,\tarm}+1)\right) \nonumber\\
& = & \left(\sum_i \sum_{N=1}^{k_{T,i}+1} 4 \rad(\emp{\mcmS}_{N,ji}, N)\right)\nonumber\\
& \le & 4\left(\sum_i (k_{T,i}+1)\rad(1, k_{T,i}+1)\right)\nonumber\\
%& \le & 4\sum_i \sum_{N=1}^{k_{T,i}+1} \left(\sqrt{\frac{\crad}{N}} + \frac{\crad}{N}\right)\\
& \le & O(\sqrt{\crad m T}) .
\end{eqnarray}
where we used $\emp{\mcm}_{N,i}$ to denote the empirical average for $i^{th}$ arm over its past $N-1$ plays. 
In the last inequality, we used Lemma \ref{lem:sumrad} along with the observation that $\sum_{i=1}^m k_{T,i} = T$. Equation \eqref{eq:tmp1}, \eqref{eq:tmp2}, and \eqref{eq:tmp3} together give
\begin{eqnarray*}
||\sumT_{t=1}^T \cv_t-\mcm \dis_t|| & \le &  O(\normOned\sqrt{\crad m T}).
\end{eqnarray*}
\end{proof}
%---------------------------------------------------------------------------------
\begin{proof}[of Lemma \ref{lem:ucbImplementability}]
 We use Fenchel duality to derive an equivalent expression for $f$: $f(\x) = \min_{||\thetaV||_*\le L} f^*(\thetaV) -\thetaV\cdot\x$ (refer to Section \ref{sec:Fenchel}). Then,
$$\psi(\dis)=\max_{\mrme \in \HC_t}\min_{||\thetaV||_*\le L} f^*(\thetaV) -\thetaV \cdot (\mrme \dis) = \min_{||\thetaV||_*\le L}  f^*(\thetaV) - \min_{\mrme \in \HC_t} \thetaV \cdot (\mrme \dis),$$
by application of the minimax theorem.
%Now, due to the nature of the set $\HC_t$, it is simple to solve the inner linear minimization problem on the right hand side for any given $\thetaV$ and $\dis$. In fact, irrespective of what $\dis$ is, for any given $\thetaV$, the inner the inner minimizer $\mrme$ is a fixed vector $\mrme(\thetaV)$ defined as:
%\textcolor{red}{put a separate definition before this lemma.}

Now, due to the structure of set $\HC_t$, observe that for any given $\thetaV$, a vertex $\vertex{\thetaV}$ (as defined in Equation \eqref{eq:MinEst}) of $\HC_t$ minimizes $\thetaV \cdot \mrme$ componentwise. Therefore, irrespective of what $\dis$ is, 
$$\psi(\dis) = \min_{||\thetaV||_*\le L} f^*(\thetaV) -\thetaV \cdot (\vertex{\thetaV}\dis),$$
which is a concave function, and a subgradient of this function at a point $\dis$ is $-{\thetaV'}^T\vertex{\thetaV'}$, where $\thetaV'$ is the minimizer of the above expression.
The minimizer
$$\thetaV'=\arg \min_{||\thetaV||_*\le L} \left(\max_{\mrme \in \HC_t} f^*(\thetaV) -\thetaV \cdot (\mrme \dis) \right)$$
is computable (e.g., by ellipsoid method) because it minimizes a convex function in $\thetaV$, with subgradient $\partial f^*(\thetaV)  - \vertex{\thetaV} \dis$ at point $\thetaV$.

The same analysis can be applied for $g(\dis)$, by using $f(\x) = -d(\x,S)$.
\end{proof}

%%%%%%%%%%%%%%%%%%%%%%%%%%%%%%%%%%%%%%%%%%%%%%%%%%%%%%%%%%%%%%%%
\subsection{Linear contextual Bandits}
\label{app:linUCB}
It is straightforward to extend Algorithm \ref{algo:UCB-BwCR} to linear contextual bandits, using existing work on UCB family of algorithms for this problem. 
%In the linear contextual version of scalar bandits problem \cite{Auer2002}, every arm $i$ is associated with a context vector $\cx_{i}$. And, there is an underlying $n$ dimensional weight vector $\wt$ such that the mean reward for arm $i$, $\mu_i=\cx_{i}\cdot \wt$. The context vectors could also vary with time, but for simplicity of illustration let us consider static contexts. 
Recall that in the contextual setting a $n$-dimensional context vector $\cx_{ji}$ is associated with every arm $i$ and component $j$, and there is an unknown weight vector $\wt_j$ for every component $j$, such that $\mcmS_{ji}=\cx_{ji} \cdot \wt_j$. 
%Then, using results on least square estimation, with high probability, the actual weight vector $\wt_j$ can be guaranteed to lie in the 
Now, consider the following ellipsoid defined by inverse of Gram matrix at time $t$,
\EQ{0in}{0in}{{\cal E}_j(t) = \{\x : (\x-\emp{\wt}_j(t))^T \B_j(t) (\x-\emp{\wt}_j(t)) \le n\},}
where 
$$\B_j(t)={\bs I}_n + \sum_{s=1}^{t-1} \cx_{j\sarm}\cx_{j\sarm}^T, \text{ and }\emp{\wt}_j(t) = \B_j(t)^{-1} \sum_{s=1}^{t-1} \cx_{j\sarm} \cvS_{s,j},$$ 
for $j=1,\dots, d$.
Results from existing literature on linear contextual bandits \cite{oful, Chu2011, Auer2003} provide that with high probability, the actual weight vector $\wt_j$ is guaranteed to lie in this elliposid, i.e.,
			$$\wt_j \in {\cal E}_j(t).$$
This allows us to define new estimate set $\HC_t$ as 
\EQ{0in}{0in}{\HC_t =\{\mcme: \mcmeS_{ji} = \cx_{ji} \cdot \tilde{\wt}_j,   \forall \tilde{\wt}_j \in {\cal E}_j(t)\}.}
Then, using results from the above-mentioned literature on linear contextual bandits, it is easy to show that the properties (1) and (2) in Equation \eqref{eq:prop1} and \eqref{eq:prop2} hold with high probability for this $\HC_t$ with $\diff(T)=\normOned n\sqrt{T \log(\frac{dT}{\delta})}$. Therefore, simply substituing this $\HC_t$ in Algorithm \ref{algo:UCB-BwCR} provides an algorithm for linear contextual version of \BwCR~, with regret bounds,
\EQ{0in}{0in}{\areg_1(T) \le  O(L \normOned n\sqrt{\frac{1}{T}\log(\frac{dT}{\delta})}), \text{~and,~} \areg_2(T) \le O(\normOned n\sqrt{\frac{1}{T} \log(\frac{dT}{\delta})}).}
%an $\regB$-optimal algorithm for linear contextual version of \BwCR~ with $\regB=\normOned n\sqrt{T\ln(\frac{dT}{\delta})}$.
%%%%%%%%%%%%%%%%%%%%%%%%%%%%%%%%%%%%%%%%%%%%%%%%%%%%%%%%%%%%%%55
\subsection{\BwK}
\label{app:UCB-BwK}

Property (1) for \BwK (stated in Equation \eqref{eq:prop1-BwK}), is simply a special case of Property (1) for \BwCR, which was proven in the previous subsection.
The following two lemmas prove the Property (2) for \BwK (stated as Equation \eqref{eq:prop2-BwK}).
The proofs are similar to the proof of Property (2) for \BwCR~ illustrated in the previous section, except that a little more careful analysis is done to get the bounds in terms of problem dependent parameters $B$ and $\OPT$.

%--------------------------------------------------------------------------
\begin{lemma}
With probability $1-(mT)e ^{-\Omega(\crad)}$,
$$||\frac{1}{T} \sumT_{t=1}^T (\rs_{t} - \UCB_{t}(\mrv) \cdot \dis_t|| \le  \sqrt{\crad m (\sum_t \rs_t)}+{\crad m}. $$
\end{lemma}
\begin{proof}
Similar to Equation \eqref{eq:tmp1} and Equation \eqref{eq:tmp2}, we can apply the concentration bounds given by Lemma \ref{lem:concentration} to get that with probability $1-(mT)e^{-\Omega(\crad)}$,
\begin{eqnarray}
|\frac{1}{T} \sumT_{t=1}^T (\rs_{t} - \mrvS_{\tarm})| & \le & 3 \rad(\frac{1}{T}\sumT_{t=1}^T \mrvS_{\tarm}, T) \nonumber\\
& \le & 3 \rad(\frac{1}{T}\sumT_{t=1}^T \UCB_{t,\tarm}(\mrv), T) \label{eq:R1}\\
|\frac{1}{T} \sumT_{t=1}^T (\UCB_{t}(\mrv) \cdot \dis_t - \UCB_{t,\tarm}(\mrv))| & \le & \rad(\frac{1}{T}\sumT_{t=1}^T \UCB_{t,\tarm}(\mrv), T) \label{eq:R2}
\end{eqnarray}
Also, using Lemma \ref{lem:empConc},
\begin{eqnarray}
\label{eq:R3}
|\sumT_{t=1}^T (\mrvS_{\tarm} - \UCB_{t,\tarm}(\mrv))| & \le & 4\sumT_t \rad(\emp{\mrvS}_{t,\tarm}, k_{t,\tarm}+1) \nonumber\\
& \le & 12\sumT_{t} \rad({\mrvS}_{\tarm}, k_{t,\tarm}+1) \nonumber\\
& = & 12 \sum_i \sum_{N=1}^{k_{T,i}+1} \rad(\mrvS_{i},N) \nonumber\\
& \le & 12 \sum_i (k_{T,i}+1)\rad(\mrvS_{i}, k_{T,i}+1) \nonumber\\
&\le & 12 \sqrt{\crad m\left(\sum_i \mrvS_{i} (k_{T,i}+1)\right)} + 12\crad m\nonumber\\
\text{(using Lemma \ref{lem:sumrad}) } &\le & 12 \sqrt{\crad m\left(\sum_t \mrvS_{\tarm} \right)} + 24 \crad m \nonumber\\
&\le & 12 \sqrt{\crad m\left(\sum_t \UCB_{t,\tarm}(\mrv) \right)} + 24\crad m
%& \le & 8 \sqrt{\crad m A} + 8\log(T) \crad m\\
%&\le & 8 m \log(T) \cdot \rad(\sum_i C_{ij}(t) (k_i(T)+1), m)+ 8\log(T) \crad m
\end{eqnarray}
Let $A=\sum_{t=1}^T \UCB_{t,\tarm}(\mrv)$. Then, from \eqref{eq:R1} and \eqref{eq:R3}, we have that for some constant $\alpha$
$$ A-2\sqrt{\alpha\crad m A} \le \sum_{t=1}^T  \rs_t + O(\crad m).$$
which implies 
$$ (\sqrt{A}-\sqrt{\alpha\crad m})^2 \le \sum_{t=1}^T  \rs_t + O(\crad m).$$
Therefore,
%$$ \sqrt{A} \le \sqrt{\alpha \crad m} +\sqrt{\sum_{t=1}^T  \rs_t} + O(\sqrt{\log(T) \crad m}).$$
%That is,
\begin{equation}
\label{eq:R4}
\sqrt{\sum_{t=1}^T \UCB_{t,\tarm}(\mrv)} \le \sqrt{\sum_{t=1}^T  \rs_t} + O(\sqrt{\crad m}).
\end{equation}
Substituting \eqref{eq:R4} in \eqref{eq:R1}, \eqref{eq:R2}, \eqref{eq:R3}, we get
\EQ{0in}{0in}{|\sumT_{t=1}^T (\rs_{t}-\sumT_{t=1}^T \UCB_{t}(\mrv) \dis_t)| \le O(\sqrt{\crad m (\sum_{t=1}^T \rs_t)}+{\crad m}).}
\end{proof}

%---------------------------------------------------------------------------
\begin{lemma}
With probability $1-(mTd)e^{-\Omega(\crad)}$, for all $j=1,\ldots, d$,
$$|\sumT_{t=1}^T (\BcvS_{t,j} - \LCB_{t,j}(\Bmcm) \dis_t| \le \sqrt{\crad m B}+{\crad m}. $$
\end{lemma}
\begin{proof}
Similar to Equation \eqref{eq:tmp1} and Equation \eqref{eq:tmp2}, we can apply the concentration bounds given by Lemma \ref{lem:concentration} to get that with probability $1-(mTd)e^{-\Omega(\crad)}$, for all $j$
\begin{eqnarray}
|\frac{1}{T} \sumT_{t=1}^T (\BcvS_{t,j} - \BmcmS_{j\tarm})| & \le & 3 \rad(\frac{1}{T}\sumT_{t=1}^T \BmcmS_{j\tarm}, T) \label{eq:B1}\\
|\frac{1}{T} \sumT_{t=1}^T (\LCB_{t,j}(\Bmcm) \dis_t - \LCB_{t,j\tarm}(\Bmcm))| & \le & \rad(\frac{1}{T}\sumT_{t=1}^T \LCB_{t,j\tarm}(\Bmcm), T) \nonumber\\
& \le & \rad(\frac{1}{T}\sumT_{t=1}^T \BmcmS_{j\tarm}, T) \label{eq:B2}
%& \le & 3 \rad(\frac{B}{T}, T)  = \frac{3}{T}\sqrt{\crad B} + \frac{\crad}{T} \label{eq:B2}
\end{eqnarray}
Also, using Lemma \ref{lem:empConc},
\begin{eqnarray}
\label{eq:B3}
|\sumT_{t=1}^T (\BmcmS_{j\tarm} - \LCB_{t,\tarm}(\Bmcm))| & \le & 4\sumT_t \rad(\emp{\BmcmS}_{t,j\tarm}, k_{t,\tarm}+1) \nonumber\\
& \le & 12 \sumT_{t} \rad({\BmcmS}_{j\tarm}, k_{t,\tarm}+1)\nonumber\\
& = & 12 \sum_i \sum_{N=1}^{k_{T,i}+1} \rad(\BmcmS_{ji},N) \nonumber\\
& \le & 12 \sum_i (k_{T,i}+1)\rad(\BmcmS_{ji}, k_{T,i}+1) \nonumber\\
&\le & 12 \sqrt{\crad m\left(\sum_i \BmcmS_{ji} (k_{T,i}+1)\right)} + 12 \crad m\nonumber\\
&\le & 12 \sqrt{\crad m\left(\sum_t \BmcmS_{j\tarm} \right)} + 24\crad m
%& \le & 8 \sqrt{\crad m A} + 8\log(T) \crad m\\
%&\le & 8 m \log(T) \cdot \rad(\sum_i C_{ij}(t) (k_i(T)+1), m)+ 8\log(T) \crad m
\end{eqnarray}
Let $A=\sum_{t}  \BmcmS_{j\tarm}$. Then, from \eqref{eq:B2} and \eqref{eq:B3}, we have that for some constant $\alpha$
$$A \le \sum_t \LCB_t(\Bmcm) \dis_t + 2\sqrt{\alpha \crad m A} + O(\crad m) \le B +  2\sqrt{\alpha \crad m A} + O(\crad m),$$
where we used that $ \sum_{t=1}^T \LCB_t(\Bmcm) \dis_t \le B$, which is a corollary of the choice of $\dis_t$ made by the algorithm.  Then,
$$(\sqrt{A} - \sqrt{\alpha \crad m})^2\le B+ O(\crad m).$$
%$$ \Rightarrow \sqrt{A} \le  \sqrt{B} + \sqrt{12 \log(T) \crad m} + 2\sqrt{\crad m} \le \sqrt{B} + 6 \sqrt{\log(T)} \sqrt{\crad m}$$
That is,
\begin{equation}
\label{eq:B4}
\sqrt{\sum_{t}  \BmcmS_{j\tarm}} \le \sqrt{B} + O(\sqrt{\crad m}).
\end{equation}
Substituting \eqref{eq:B4} in \eqref{eq:B1}, \eqref{eq:B2}, \eqref{eq:B3}, we get
$$|\sumT_{t=1}^T (\Bcv_{t,j}-\sumT_{t=1}^T \LCB_{t,j}(\Bmcm) \dis_t)| \le O(\sqrt{\crad m B}+{\crad m}).$$
\end{proof}

%==========================================================================
\section{Frank-Wolfe}
\label{app:frankWolfe}

%--------------------------------------------------------------------------------------

\begin{proof}[of Theorem \ref{th:primal}]
Let  $\Delta_t := f(\optxv) -f(\avgt \x)$.  
We prove that $ \Delta_t \le \frac{\sC \log(2t)}{2t}$.  (The base of the $\log$ is 2.)
Once again, we use (\ref{eq:cor2}) for $t+1$, with $\y_{t+1} = \avgt \x$, and rearrange terms as follows: 
\begin{equation}
\label{eq:FW1} \nabla f(\avgxvt) \cdot(\x_{t+1}-\avgxvt)  \geq f(\optxv)-f(\avgxvt). 
\end{equation} 

In order to  use (\ref{eq:cfq}), we rewrite $\avgxvtp $   $=  \avgxvt + \frac{1}{t+1}(\xv_{t+1}-\avgxvt) $. 
Using  (\ref{eq:cfq}) first, followed by (\ref{eq:FW1}), gives us the following. 
\begin{eqnarray*}
f(\avgxvtp) & \ge & f(\avgxvt) + \frac{1}{t+1} \nabla f(\avgxvt) \cdot(\xv_{t+1}-\avgxvt) - \frac{\sC}{2(t+1)^2}\\
& \ge & f(\avgxvt) + \frac{1}{t+1} (f(\optxv)-f(\avgxvt)) - \frac{\sC}{2(t+1)^2}
\end{eqnarray*}
With this we can bound $\Delta_{t+1}$ in terms of $\Delta_t$. 
\begin{equation} 
\label{eq:deltadecrease}
 \Delta_{t+1}  \le  \Delta_t - \frac{1}{(t+1)}\Delta_t + \frac{\sC}{2(t+1)^2}
= \frac{t}{(t+1)}\Delta_t + \frac{\sC}{2(t+1)^2}
\end{equation} 
 Recall that we wish to show that $ \Delta_t \le \sC \log(2t)/{2t}$. 
The rest of the proof is by induction on $t$. 
For the base case, we note that  we can still use (\ref{eq:deltadecrease}) with  $t=0$
and an arbitrary $\xv_0$ which is used to set $\dis_1$. This gives us that
$ \Delta_1 \le \sC/2.$ 
%Therefore, the base case holds. 
The inductive step for $t+1$ follows from (\ref{eq:deltadecrease}) and the inductive hypothesis for $t$ if 
\begin{eqnarray*}
 \frac{t}{(t+1)} \cdot \frac{\sC\log(2t)}{2t}  + \frac{\sC}{2(t+1)^2}
&\leq&  \frac{\sC\log(2(t+1))}{2(t+1)}  \\
\Leftrightarrow \log(t)  + \frac 1 {t+1} &\leq& \log(t+1) \\
\Leftrightarrow \frac 1 {t+1} &\leq & \log(1+ \tfrac 1 t) . 
\end{eqnarray*}
The last inequality follows from the fact that  for any $a> 0$, $\log(1+a) > \frac{a}{1+a}$, by setting $a=1/t$. 
%so $\log(\frac{t}{t+1}) = -\log(1+\frac{1}{t}) < -\frac{1}{t+1}$. Substituting, we get
%\begin{eqnarray*}
%\Delta_{t+1} & < &  \frac{C}{2(t+1)} \log(2(t+1))  - \frac{C}{2(t+1)^2}  + \frac{C}{2(t+1)^2}\\
%& = &  \frac{C}{2(t+1)} \log(2(t+1))\\
%\end{eqnarray*}
This completes the induction. Therefore,
$\Delta_T  = f(\optxv) - f(\avgT \xv)  \le  \frac{\sC\log(2T)}{2T}$ and combined with (\ref{eq:regCalc2}), we get the desired theorem statement.
\end{proof}
%------------------------------------------------------------------

\begin{proof}[of Theorem \ref{th:Nesterov} ]
We first show Lipshitz continuity of $\nabla \hat{f}_\sP$. 
Let $\x_1$ and $\x_2$ be any two points in the domain of $f$, 
then for $\ell = 1,2$, $\nabla\hat{f}_\sP(\x_\ell) = - \thetaV_\ell$ where 
\[ \thetaV_\ell = \arg \min_{||\thetaV|| \leq L} \{ f^*(\thetaV) + \frac \sP {2L} ||\thetaV||^2  - \thetaV\cdot \x_\ell\} . \] 
We use the following fact about convex functions: if $\y^*$ minimizes a convex function $\psi$ 
over some domain and $\y$ is any other point in the domain then $\nabla \psi (\y^*) \cdot (\y - \y^*) \geq 0$.  
Using this fact for $\y^* = \thetaV_1$ and $\y = \thetaV_2$, we get that 
\begin{equation} 
\label{eq:nes1} 
\left(\nabla f^*(\thetaV_1) + \frac {\sP} {L} \thetaV_1 - \x_1 \right)\cdot (\thetaV_2 - \thetaV_1) \geq 0 . 
\end{equation} 
Using convexity of $f^*$ and strong convexity of $||\cdot||^2$, we get that 
\begin{eqnarray}  
\label{eq:nes2} 
f^*(\thetaV_2)  &\geq& f^*(\thetaV_1) + \nabla f^*(\thetaV_1) \cdot (\thetaV_2 - \thetaV_1),\\
\label{eq:nes3} \frac \sP {2L} ||\thetaV_2||^2 &\geq& \frac \sP {2L} \left(   ||\thetaV_1||^2 + 2\thetaV_1\cdot (\thetaV_2- \thetaV_1) 
+ ||\thetaV_2 - \thetaV_1||^2  \right ). 
\end{eqnarray}
Adding (\ref{eq:nes1}--\ref{eq:nes3}) we get that 
\[ - \x_1 \cdot (\thetaV_2 - \thetaV_1)+f^*(\thetaV_2) + \frac \sP {2L} ||\thetaV_2||^2 
\geq   f^*(\thetaV_1)+\frac \sP {2L} \left(   ||\thetaV_1||^2+||\thetaV_2 - \thetaV_1||^2  \right ). \]
Similarly, by switching $\x_1$ and $\x_2$, we get 
\[ - \x_2 \cdot (\thetaV_1 - \thetaV_2)+f^*(\thetaV_1) + \frac \sP {2L} ||\thetaV_1||^2 
\geq   f^*(\thetaV_2)+\frac \sP {2L} \left(   ||\thetaV_2||^2+||\thetaV_2 - \thetaV_1||^2  \right ). \]
Adding these two, we get 
\[  (\x_1- \x_2) \cdot (\thetaV_1 - \thetaV_2)  \geq \frac \sP L ||\thetaV_2 - \thetaV_1||^2.\]
By Caucy-Schwartz inequality, we have 
\begin{eqnarray*} 
(\x_1- \x_2) \cdot (\thetaV_1 - \thetaV_2) & \leq  & ||\x_1- \x_2|| \cdot||\thetaV_1 - \thetaV_2|| \\
\therefore \frac \sP L ||\thetaV_2 - \thetaV_1||^2 & \leq  & ||\x_1- \x_2|| \cdot||\thetaV_1 - \thetaV_2|| \\
\Rightarrow  ||\thetaV_2 - \thetaV_1|| & \leq  &\frac L \sP  ||\x_1- \x_2||.  
\end{eqnarray*} 
 This shows that the Lipschitz constant of $\nabla \hat{f}_\sP$ is $L/\sP$.

Then, we can show that  $\hat{f}_\sP$ is $\frac{dL}{\sP}$ smooth as follows:
\begin{eqnarray*}
\hat{f}_\sP(\x+\alpha(\y-\x)) & = & \hat{f}_\sP(\x) - \int_{w: 0}^\alpha \nabla \hat{f}_\sP(\x+w(\y-\x)) \cdot (\y-\x) dw \\
& = & \hat{f}_\sP(\x) + \alpha \nabla\hat{f}_\sP(\x)(\y-\x) + \int_{w: 0}^\alpha(\nabla \hat{f}_\sP(\x+w(\y-\x))- \nabla \hat{f}_\sP(\x))\cdot (\y-\x) dw\\
\end{eqnarray*}
Then, using Lipschitz continuity of $\hat{f}_\sP$,
\begin{eqnarray*}
\left|\int_{w: 0}^\alpha(\nabla \hat{f}_\sP(\x+w(\y-\x))- \nabla \hat{f}_\sP(\x))\cdot (\y-\x) dw\right| & \le & \frac{L\alpha^2}{\sP} ||\x-\y|| \cdot ||\x-\y|| \int_{0}^\alpha (w)dw\\
& = &  \frac{L\alpha^2}{2\sP} ||\x-\y||^2\\
%& \le &  \frac{L\alpha^2}{2\sP} ||{\bf 1}_d|| ||{\bf 1}_d|| \\
& \le &  \frac{dL}{\sP} \cdot \frac {\alpha^2} 2
% the last term is \ge -G||x-y|| ||x-y||_* \int_t (tdt)
% And ||x-y|| ||x-y||_* \le ||1_d|| ||1_d||_* \le d
\end{eqnarray*}
%this holds irrespective of the norm with respect to which gradient of f is Lipschitz continuous, when the domain is [0,1]^d because ||1_d||\cdot ||1_d||_*=d.
It remains to show that $\hat{f}_\sP -\frac{\sP L}{2} \leq f \leq \hat{f}_\sP.$ This follows almost immediately from 
Lemma \ref{lem:FenchelDuality} and  (\ref{eq:fsmooth2}): the function inside the minimization for 
$\hat{f}_\sP$ is always larger than that of $f$, but not by more than $\frac{\sP L}{2}$.  
\end{proof} 

%---------------------------------------------------------------------------------

\begin{lemma} \label{lem:fmu} 
$\nabla \hat{f}_\sP (\z) = -\thetaV$ iff $\exists~\y$ such that 
\begin{enumerate} 
\item $-\thetaV$ is a supergradient of $f$ at $\y$. We denote this by $-\thetaV \in \partial f (\y) $, and 
\item $-\thetaV= \alpha (\y -\z) $  where $\alpha = \min \{  L/\sP, L/||\y - \z || \}$. 
\end{enumerate} 
\end{lemma} 
\begin{proof} 
The gradient of $\hat{f}_\sP$ is equal to $-\thetaV$ where $\thetaV$ is the $\arg \min$ in \eqref{eq:fsmooth2}, which is equivalent to 
\[ \min_{||\thetaV|| \leq L} \max_{\y}  \{ f(\y) + \thetaV\cdot \y  + \frac{\sP}{2L} \thetaV \cdot \thetaV  - \thetaV\cdot \z\} 
= \max_{\y} \min_{||\thetaV|_* \leq L}  \{ f(\y) + \thetaV\cdot \y  + \frac{\sP}{2L} \thetaV \cdot \thetaV  - \thetaV\cdot \z\} ,\]
by the min-max theorem. 
The two conditions in the hypothesis of the lemma are essentially the KKT conditions for the above.
Given $\thetaV$, it must be that $\y$ optimizes the inner maximization in the first form, which happens when  $-\thetaV \in \partial f(\y) $.  On the other hand, given $\y$, it must be that $\thetaV$ optimizes the inner minimization in the second form. 
Note that due to the spherical symmetry of the domain of $\thetaV$, the {\em direction} that minimizes is 
$\z-\y$. Therefore we may assume that $\thetaV = \alpha (\z -\y)$ for some $0 \leq \alpha \leq L /||\z - \y||$, since 
$||\thetaV||\leq L $.    Given this, the inner minimization reduces to minimizing $ \sP \alpha^2/(2L)
 - \alpha$, subject to the constraint on $\alpha$ above, the solution to which is  $\alpha = \min \{  L/\sP, L/||\y - \z || \}$. 
\end{proof} 

%========================================================
%\label{sec:distance-fw}
\begin{proof}[of Lemma \ref{lem:distance-smooth}]
%Now consider the case that $f(\x) = -d(\x,S)$, which corresponds to the \BwC problem. 
%We obtain the following approximation $\hat{d}_\sP(\cdot)$ for the distance function $d(\cdot)$ from \eqref{eq:fsmooth2}:
%$$ \hat{d}_\sP(\x,S) = \max_{||\thetaV|| \le 1}\thetaV\cdot \x - h_S(\thetaV) - \frac{\sP}{2} \thetaV\cdot \thetaV,$$
%where $h_S(\thetaV)=\max_{\s \in S} \thetaV \cdot \s$, as defined earlier.
One can get a closed form expression for the subgradients of the distance function. 
Let $\proj_S(\z)$ be the projection of $\z$ onto $S$ for $\z \notin S$, 
and $\nu_S(\z)$ be the set of unit normal vectors to $S$ at $\z$, for a $\z$ that is on the boundary of $S$.
We extend the defintion of $\nu_S(\z)$ to $\z \notin S$ as
\[ \nu_S(\z) := \frac {\z - \proj_S(\z)} {||\z - \proj_S(\z)||} .\]
Then, the set of subgradients of the distance function $\partial d(\z, S)$ is as follows. 
\begin{equation*}
\partial d(\z, S) =  \threepartdef 
	{\nu_S(\z) }  	{ \z \notin S} 
		{\{\alpha \nu_S(\z) , \text{ for all } \alpha \in [0,1]\}}	{\z \text{ is on the boundary of } S}	
	{\mathbf 0} 									{\z \in \text{ interior of } S} 
\end{equation*} 
Note that $d(\cdot,S)$ is non-smooth near the boundary of $S$. We now show how $\hat{d}_\sP(\cdot, S)$ becomes smooth, and give the stated closed form expression for $\nabla  \hat{d}_\sP(\cdot ,S)$. 

%\[ \nabla \hat{d}_\sP (\z) =  \threepartdef 
%	{\nu_S(\z) }  	{ ||\z  -\pi_S(\z)|| \geq \sP } 
%	{\nu_S(\z) \frac {||\z  -\pi_S(\z)||}  \sP}			{ 0 < ||\z  -\pi_S(\z)|| < \sP}
%	{\mathbf 0} 			{\z \in  S} 
%\] 

We use Lemma \ref{lem:fmu} for $f(\z) = -d(\z,S)$ to construct for each $\z$, a $\y$ that satisfies the two conditions in the lemma, and gives 
$\nabla \hat{f}_\sP(\z) = -\nabla \hat{d}_\sP(\z, S)$ as claimed. Note that $L=1$ in this case. 
\paragraph{Case 1: $	{ ||\z  -\pi_S(\z)|| \geq \sP } $} 
Pick  $\y = \pi_S(\z)$.  Note that $\nu_S(\z) \in \nu_S(\y)$ therefore $-\nu_S(\z) \in\partial f(\y)$, and the first condition in 
Lemma \ref{lem:fmu} is satisfied. Since $||\z - \y|| \geq \sP$, $\alpha = 1/||\z - \y|| $ and 
$\alpha (\y - \z) = - \nu_S(\z)$, so the second condition in Lemma \ref{lem:fmu} is satisfied.

\paragraph{Case 2: ${ 0 < ||\z  -\pi_S(\z)|| < \sP}$} 
Pick  $\y = \pi_S(\z)$.  As in Case 1, $\nu_S(\z) \in \nu_S(\y)$ therefore 
$-\nu_S(\z)\tfrac {||\z  -\pi_S(\z)||}  \sP \in\partial f(\y)$, and 
the first condition in  Lemma \ref{lem:fmu} is satisfied. 
Since $||\z - \y|| < \sP$, $\alpha = 1/\sP $ and $\alpha (\y - \z) = \frac{\pi_S(\z) - \z} \sP$
so  the second condition in Lemma \ref{lem:fmu} is satisfied. 

\paragraph{Case 3: ${\z \in  S} $} 
Pick $\y = \z$. Note that ${\mathbf 0} \in\partial f(\y)$, and the conditions in Lemma \ref{lem:fmu} are satisfied trivially. 
 
\end{proof}